%% file: main.tex
\documentclass[manuscript, screen, nonacm]{acmart}

\AtBeginDocument{%
  \providecommand\BibTeX{{%
    \normalfont B\kern-0.5em{\scshape i\kern-0.25em b}\kern-0.8em\TeX}}}

\copyrightyear{2021}

\usepackage{nicefrac}       
\usepackage{enumitem}
\usepackage{tabularx}
\usepackage{algpseudocode}
\usepackage{algorithm}
\usepackage{subcaption}

\renewcommand\vec{\mathbf}
\newcommand{\dd}{\mathrm{d}} 
\newcommand{\var}{\operatorname{var}}

\newcommand{\E}{\mathbb{E}}
\newcommand{\reals}{\mathbb{R}}
\newcommand{\ind}[1]{\mathbf{1}_{#1}}
\newcommand{\Xspace}{\mathcal{X}}
\newcommand{\Yspace}{\mathcal{Y}}
\newcommand{\risk}{R}
\newcommand{\Jac}{\mathrm{D}}
\newcommand{\trace}{\operatorname{tr}}

\newcommand{\testset}{\mathcal{T}}

\newtheorem{remark}{Remark}
\newtheorem{corollary}{Corollary}
\newtheorem{proposition}{Proposition}
\newtheorem{definition}{Definition}
\newtheorem{example}{Example}
\newtheorem{theorem}{Theorem}

\graphicspath{{./figures/}}

\begin{document}
\fancyhead{}

\title{Needle in a Haystack: Label-Efficient Evaluation under Extreme Class Imbalance}

\author{Neil G.~Marchant}
\email{nmarchant@unimelb.edu.au}
\orcid{0000-0001-5713-4235}
\affiliation{%
  \institution{School of Computing and Information Systems, University of Melbourne}
  \streetaddress{}
  \city{Melbourne}
  \state{Victoria}
  \country{Australia}
  \postcode{3010}
}
\author{Benjamin I.~P.~Rubinstein}
\email{brubinstein@unimelb.edu.au}
\orcid{0000-0002-2947-6980}
\affiliation{%
  \institution{School of Computing and Information Systems, University of Melbourne}
  \streetaddress{}
  \city{Melbourne}
  \state{Victoria}
  \country{Australia}
  \postcode{3010}
}


\begin{abstract}
  Important tasks like record linkage and extreme classification 
  demonstrate extreme class imbalance, with 1 minority 
  instance to every 1 million or more majority instances. 
  Obtaining a sufficient sample of all classes, even just to achieve 
  statistically-significant evaluation, is so challenging that most current 
  approaches yield poor estimates or incur impractical cost. 
  Where importance sampling has been levied against this challenge, 
  restrictive constraints are placed on performance metrics, estimates do 
  not come with appropriate guarantees, or evaluations 
  cannot adapt to incoming labels.
  This paper develops a framework for online evaluation based 
  on adaptive importance sampling. 
  Given a target performance metric and model for $p(y|x)$, the 
  framework adapts a distribution over items to label in order to 
  maximize statistical precision. 
  We establish strong consistency and a central limit theorem for the 
  resulting performance estimates, and instantiate our framework with 
  worked examples that leverage Dirichlet-tree models.
  Experiments demonstrate an average MSE superior to state-of-the-art 
  on fixed label budgets.
\end{abstract}




\maketitle

\input{subfiles/intro}
\input{subfiles/problem}
\input{subfiles/ais}
\input{subfiles/experiments}
\input{subfiles/related}
\input{subfiles/conclusion}

\begin{acks}
  \begin{sloppypar}
  N.~Marchant acknowledges the support of an Australian Government 
  Research Training Program Scholarship. 
  B.~Rubinstein acknowledges the support of Australian Research Council grant 
  DP150103710.
  \end{sloppypar}
\end{acks}

\bibliographystyle{ACM-Reference-Format}
\bibliography{ais}

\appendix

\input{subfiles/appendix.tex}

\end{document}

%% file: subfiles/intro.tex
\section{Introduction}
Evaluation of machine learning systems under extreme class imbalance seems 
like a hopeless task. 
When minority classes are exceedingly rare---e.g.\ occurring at a rate of one 
in a million---a massive number of examples (1~million in expectation) 
must be labeled before a single minority example is encountered. 
It seems nigh impossible to reliably estimate performance in these 
circumstances, as the level of statistical noise is simply too high. 
Making matters worse, is the fact that high quality labels for evaluation are 
rarely available for free. 
Typically they are acquired manually at some cost---e.g.\ by employing 
expert annotators or crowdsourcing workers. 
Reducing labeling requirements for evaluation, while ensuring estimates of 
performance are precise and free from bias, is therefore paramount.
One cannot afford to waste resources on evaluation if the results
are potentially misleading or totally useless. 

From a statistical perspective, evaluation can be cast as the estimation
of population performance measures using independently-drawn test data. 
Although \emph{unlabeled} test data is abundant in many applied settings, 
labels must usually be acquired as part of the evaluation process.
To ensure estimated performance measures converge to their population values, 
it is important to select examples for labeling in a statistically sound 
manner. 
This can be achieved by sampling examples \emph{passively} according to 
the data generating distribution. 
However, passive sampling suffers from poor label efficiency for 
some tasks, especially under extreme class imbalance.
This impacts a range of application areas including fraud detection 
\citep{wei_effective_2013}, record linkage \citep{marchant_search_2017}, 
rare diseases \citep{khalilia_predicting_2011} 
and extreme classification \citep{schultheis2020unbiased}.

The poor efficiency of passive sampling for some evaluation tasks motivates 
\emph{active} or \emph{biased sampling} strategies, which improve efficiency 
by focusing labeling efforts on the ``most informative'' 
examples~\citep{sawade_active_2010a}. 
Previous work in this area is based on variance-reduction methods, 
such as stratified sampling \citep{bennett_online_2010, 
druck_toward_2011, gao_efficient_2019}, importance sampling 
\citep{sawade_active_2010, schnabel_unbiased_2016} and adaptive 
importance sampling \citep{marchant_search_2017}. 
However existing approaches suffer from serious limitations, including 
lack of support for a broad range of performance measures 
\citep{bennett_online_2010, sawade_active_2010, welinder_lazy_2013, 
schnabel_unbiased_2016, marchant_search_2017}, weak theoretical 
justification~\citep{bennett_online_2010, druck_toward_2011, 
welinder_lazy_2013} and an inability to adapt sampling based on incoming 
labels~\citep{sawade_active_2010, schnabel_unbiased_2016}. 

In this paper, we present a general framework for label-efficient online 
evaluation that addresses these limitations.
Our framework supports any performance measure that can be expressed as a 
transformation of a vector-valued risk functional---a much broader 
class than previous work. 
This allows us to target simple scalar measures such as accuracy and F1 score, 
as well as more complex multi-dimensional measures such as precision-recall
curves for the first time.
We leverage adaptive importance sampling (AIS) to efficiently select 
examples for labeling in batches.
The AIS proposal is adapted using labels from previous batches in order to 
approximate the asymptotically-optimal variance-minimizing proposal. 
This approximation relies on online estimates of $p(y|x)$, which we propose 
to estimate via a Bayesian Dirichlet-tree~\citep{dennis_bayesian_1996} 
model that achieves asymptotic optimality for deterministic labels. 

We analyze the asymptotic behavior of our framework under general conditions, 
establishing strong consistency and a central limit theorem. 
This improves upon a weak consistency result obtained in a less general 
setting \citep{marchant_search_2017}. 
We also compare our framework empirically against four baselines: 
passive sampling, stratified sampling, importance sampling 
\citep{sawade_active_2010}, and the stratified AIS method of 
\citep{marchant_search_2017}. 
Our approach based on a Dirichlet-tree model, achieves 
superior or competitive performance on all but one of seven test cases.
Proofs and extensions are included in appendices.
A Python package implementing our framework has been released open-source 
at \url{https://github.com/ngmarchant/activeeval}.

%% file: subfiles/problem.tex
\section{Preliminaries}
\label{sec:preliminaries}
We introduce notation and define the label-efficient evaluation problem 
in Section~\ref{sec:problem}. Then in Section~\ref{sec:gen-measures}, we 
specify the family of performance measures supported by our framework.
Section~\ref{sec:limitations-mc} presents novel insights into the 
impracticality of passive sampling relative to class imbalance and evaluation 
measure, supported by asymptotic analysis. 

\subsection{Problem formulation}
\label{sec:problem}
Consider the task of evaluating a set of systems $\mathcal{S}$ which solve 
a prediction problem on a feature space $\Xspace \subseteq \reals^m$ 
and label space $\Yspace \subseteq \reals^l$. 
Let $f^{(s)}(x)$ denote the output produced by system $s \in \mathcal{S}$ 
for a given input $x \in \Xspace$---e.g.\ a predicted label or  
distribution over labels.
We assume instances encountered by the systems are generated i.i.d.\ from 
an unknown joint distribution with density $p(x, y)$ on 
$\Xspace \times \Yspace$. 
Our objective is to obtain \emph{accurate} and \emph{precise} estimates of 
target performance measures (e.g.\ F1 score) with respect to $p(x, y)$ 
at minimal cost. 

We consider the common scenario where an \emph{unlabeled} test pool 
$\testset = \{x_1, \ldots, x_M\}$ drawn from $p(x)$ is available upfront. 
We assume labels are \emph{unavailable} initially and can only 
be obtained by querying a stochastic \emph{oracle} that returns draws from 
the conditional $p(y|x)$.
We assume the response time and cost of oracle queries far outweigh 
contributions from other parts of the evaluation process. 
This is reasonable in practice, since the oracle requires human 
input---e.g.~annotators on a crowdsourcing platform or domain experts. 
Under these assumptions, minimizing the cost of evaluation is equivalent 
to minimizing the number of oracle queries required to estimate 
target performance measures to a given precision. 

\begin{remark}
  \label{rem:oracle}
  A \emph{stochastic oracle} covers the most general case where  
  $p(y|x)$ has support on one or more labels. 
  This may be due to a set of heterogeneous or noisy annotators (not 
  modeled) or genuine ambiguity in the label. 
  We also consider a \emph{deterministic oracle} where $p(y|x)$ is 
  a point mass. 
  This is appropriate when trusting individual judgments from an 
  expert annotator.
\end{remark}


\subsection{Generalized measures}
\label{sec:gen-measures}

When embarking on an evaluation task it is important to select a suitable 
measure of performance.
For some tasks it may be sufficient to measure global error rates, while 
for others it may be desirable to measure error rates for different classes, 
sub-populations or parameter configurations---the possibilities are boundless.
Since no single measure is suitable for all tasks, we consider a broad 
family of measures which correspond mathematically to transformations of 
vector-valued risk functionals.

\begin{definition}[Generalized measure]
  \label{def:gen-measure}
  Let $\ell(x, y; f)$ be a vector-valued loss function that maps instances 
  $(x, y) \in \Xspace \times \Yspace$ to vectors in $\reals^d$ 
  dependent on the system outputs $f = \{f^{(s)}\}$.
  We suppress explicit dependence on $f$ where it is understood.
  Assume $\ell(x, y; f)$ is uniformly bounded in sup norm for all 
  system outputs $f$.
  Denote the corresponding vector-valued risk functional 
  by $\risk = \E_{X, Y \sim p}[\ell(X, Y; f)]$. 
  For any choice of $\ell$ and continuous mapping 
  $g: \reals^{d} \to \reals^{m}$ differentiable at $R$, the \emph{generalized 
  measure} is defined as
  $G = g(\risk)$.
\end{definition}

\begin{table}
  \caption{Representations of binary classification measures as generalized 
  measures. 
  Here we assume $\Yspace = \{0,1\}$, $f(x)$ denotes the predicted class 
  label, and $\hat{p}_1(x)$ is an estimate of $p(y = 1|x)$ according to 
  the system under evaluation.}
  \label{tbl:classification-measures}
  \centering
  \begin{tabular}{l c c}
    \toprule
    {\normalsize Measure} & $\ell(x, y)^\intercal$ & $g(\risk)$ \\
    \midrule
    Accuracy & 
    $\ind{y \neq f(x)} \vphantom{\Big]}$ &
    $1 - \risk$ \\
    Balanced accuracy & 
    $\left[y f(x), y, f(x)\right] \vphantom{\Big]}$ &
    $\frac{\risk_1 + \risk_2 (1 - \risk_2 - \risk_3)}{2\risk_2 (1 - \risk_2)}$  \\
    Precision & 
    $\left[y f(x), f(x)\right] \vphantom{\Big]}$ & 
    $\frac{\risk_1}{\risk_2}$ \\
    Recall & 
    $\left[y f(x), y\right] \vphantom{\Big]}$ & 
    $\frac{\risk_1}{\risk_2}$ \\
    $F_\beta$ score & 
    $\left[y f(x), \frac{\beta^2 y + f(x)}{1 + \beta^2}\right]$ & 
    $\frac{\risk_1}{\risk_2}$ \\
    \parbox{0.28\linewidth}{\raggedright Matthews correlation coefficient} & 
    $\left[y f(x), y, f(x)\right] \vphantom{\Big]}$ & 
    $\frac{\risk_1 - \risk_2 \risk_3}{\sqrt{\risk_2 \risk_3 (1 - \risk_2) (1 - \risk_3)}}$ \\
    \parbox{0.28\linewidth}{\raggedright Fowlkes-Mallows index} & 
    $\left[y f(x), y, f(x)\right] \vphantom{\Big]}$ & 
    $\frac{R_1}{\sqrt{R_2 R_3}} \vphantom{\Big]}$ \\
    Brier score & 
    $2(\hat{p}_1(x) - y)^2$ &
    $R \vphantom{\Big]}$ \\
    \bottomrule
  \end{tabular}
\end{table}

\begin{table}
  \caption{Representations of scalar regression measures as generalized 
  measures. 
  Here $f(x)$ denotes the predicted response according to the 
  system under evaluation.}
  \label{tbl:regression-measures}
  \centering
  \begin{tabular}{l c c}
    \toprule
    Measure      & $\ell(x, y)^\intercal$ & $g(\risk)$ \\
    \midrule
    Mean absolute error & 
      $|y - f(x)|$ &
      $\risk \vphantom{\Big]}$\\
    Mean squared error & 
      $(y - f(x))^2$ &
      $\risk \vphantom{\Big]}$ \\
      \parbox{0.28\linewidth}{\raggedright Coefficient of determination} & 
      $[y, y^2, f(x), f(x)^2]$ & 
      $\frac{\risk_4 - 2 \risk_1 \risk_3 + \risk_1^2}{\risk_2 - \risk_1^2} \vphantom{\Big]}$ \\
    \bottomrule
  \end{tabular}
\end{table}

Although this definition may appear somewhat abstract, it encompasses a
variety of practical measures.
For instance, when $g$ is the identity and $d = 1$ the family reduces to 
a scalar-valued risk functional, which includes accuracy and mean-squared 
error as special cases. 
Other well-known performance measures such as precision and recall can 
be represented by selecting a non-linear $g$ and a vector-valued $\ell$.
Tables~\ref{tbl:classification-measures} and~\ref{tbl:regression-measures} 
demonstrate how to recover standard binary classification and regression 
measures for different settings of $g$ and $\ell$.
In addition to scalar measures, the family also encompasses vector-valued 
measures for vector-valued $g$ and $\ell$. 
These can be used to estimate multiple scalar measures simultaneously---e.g.\ 
precision and recall of a system, accuracy of several competing systems, or 
recall of a system for various sub-populations.
Below, we demonstrate that a vector-valued generalized measure can represent 
a precision-recall (PR) curve.

\begin{example}[PR curve]
  \label{ex:pr-curve}
  A precision-recall (PR) curve plots the precision and recall of a soft 
  binary classifier on a grid of classification thresholds 
  $\tau_1 < \tau_2 < \cdots < \tau_L$. 
  Let $f(x) \in \reals$ denote the classifier score for input $x$, where a 
  larger score means the classifier is more confident the label is positive 
  (encoded as `1') and a smaller score means the classifier is more confident 
  the label is negative (encoded as `0').
  We define a vector loss function that measures whether an instance 
  $(x,y)$ is: 
  (1)~a predicted positive for each threshold (the first $L$ entries), 
  (2)~a true positive for each threshold (the next $L$ entries), and\slash or 
  (3)~a positive (the last entry):
  \begin{equation*}
    \ell(x, y) = \left[ \ind{f(x) \geq \tau_1}, \ldots, \ind{f(x) \geq \tau_L}, 
      y \ind{f(x) \geq \tau_1}, \ldots, y \ind{f(x) \geq \tau_L}, y \right]^\intercal.
  \end{equation*}
  A PR curve can then be obtained using the following mapping function:
  \begin{equation*}
    G = g(R) = \left[ \frac{R_{L+1}}{R_1}, \ldots, \frac{R_{2L}}{R_L}, \frac{R_{L+1}}{R_{2L+1}}, \ldots, \frac{R_{2L}}{R_{2L+1}} \right]^\intercal,
  \end{equation*}
  where the first $L$ entries correspond to the precision at each threshold in 
  ascending order, and the last $L$ entries correspond to the recall at each 
  threshold in ascending order.
\end{example}

\begin{remark}
  We have defined generalized measures with respect to the data generating 
  distribution $p(x, y)$. 
  While this is the ideal target for evaluation, it is common in 
  practice to define performance measures with respect to a sample.
  We can recover these from our definition by substituting an empirical 
  distribution for $p(x,y)$. 
  For example, the familiar sample-based definition of recall can 
  be obtained by setting $\ell(x, y) = [y f(x), y]^\intercal$, 
  $g(R) = R_1 / R_2$ and 
  $p(x, y) = \frac{1}{N} \sum_{i = 1}^{N} \ind{x_i = x} \ind{y_i = y}$.
  Then
  \begin{equation*}
    G_\mathrm{rec} 
    = g(R) 
    = \frac{\frac{1}{N} \sum_{i = 1}^{N} y_i f(x_i)}{\frac{1}{N} \sum_{i = 1}^{N} y_i} 
    = \frac{\mathrm{TP}}{\mathrm{TP} + \mathrm{FN}}.
  \end{equation*}
  Given our assumption that the test pool $\testset$ is drawn from $p(x)$, 
  any consistent sample-based estimator will converge to the population 
  measure.
\end{remark}

\subsection{Inadequacy of passive sampling}
\label{sec:limitations-mc}

We have previously mentioned passive sampling as an obvious baseline for 
selecting instances to label for evaluation.
In this section, we conduct an asymptotic analysis for two sample evaluation 
tasks, which highlights the impracticality of passive sampling under extreme 
class imbalance. 
This serves as concrete motivation for our interest in label-efficient 
solutions.
We begin by defining an estimator for generalized measures based on 
passive samples.

\begin{definition}[Passive estimator for $G$]
  \label{def:passive}
  Let $\mathcal{L} = \{(x_1, y_1), \allowbreak \ldots, \allowbreak (x_N, y_N)\}$ 
  be a labeled sample of size $N$ drawn passively according to $p(x,y)$. 
  In practice, $\mathcal{L}$ is obtained by drawing instances i.i.d.\ from 
  the marginal $p(x)$ and querying labels from the oracle $p(y|x)$.
  The \emph{Monte Carlo} or \emph{passive} estimator for a generalized 
  measure $G$ is then defined as follows:
  \begin{equation*}
    \hat{G}_{N}^{\mathrm{MC}} = g(\hat{\risk}_{N}^{\mathrm{MC}})
    \text{ with }
    \hat{\risk}_{N}^{\mathrm{MC}} = \frac{1}{N} \sum_{(x,y) \in \mathcal{L}} \ell(x, y).
  \end{equation*}
\end{definition}

Note that $\hat{G}_{N}^{\mathrm{MC}}$ is a biased estimator for $G$ in 
general, since $g$ may be non-linear. 
However, it is asymptotically unbiased---that is, $\hat{G}_{N}^{\mathrm{MC}}$ 
converges to $G$ with probability one in the limit $N \to \infty$. 
This property is known as \emph{strong consistency} and it follows from the 
strong law of large numbers \citep[pp.~243--245]{feller_introduction_1968} 
and continuity of $g$. 
There is also a central limit theorem (CLT) for $\hat{G}_{N}^{\mathrm{MC}}$, 
reflecting the  rate of convergence: 
$\E[\|\hat{G}_{N}^{\mathrm{MC}} - G\|] \leq \| \Sigma \|/\sqrt{N}$ 
asymptotically where $\Sigma$ is an asymptotic covariance matrix  
(see Theorem~\ref{thm:ais-clt}).
We shall now use this result to analyse the asymptotic efficiency of the 
passive estimator for two evaluation tasks.

\begin{example}[Accuracy]
  Consider estimating the accuracy $G_\mathrm{acc}$ (row~1 of 
  Table~\ref{tbl:classification-measures}) of a classifier. 
  By the CLT, it is straightforward to show that the passive estimator for 
  $G_\mathrm{acc}$ is asymptotically normal with variance 
  $G_\mathrm{acc}(1 - G_\mathrm{acc}) / N$.
  Thus, to estimate $G_\mathrm{acc}$ with precision $w$ we require 
  a labeled sample of size 
  $N \propto G_\mathrm{acc}(1 - G_\mathrm{acc})/w^2$. 
  Although this is suboptimal\footnote{Theoretically it is possible to achieve 
  an asymptotic variance of zero.} it is not impractical. 
  A passive sample reasonably captures the variance in $G_\mathrm{acc}$. 
\end{example}

This example shows that passive sampling is not always a poor choice. 
It can yield reasonably precise estimates of a generalized measure, so long 
as the measure is sensitive to regions of the space $\Xspace \times \Yspace$ 
with \emph{high density} as measured by $p(x, y)$.
However, where these conditions are not satisfied, passive sampling may 
become impractical, requiring huge samples of labeled data in order to 
sufficiently drive down the variance.
This is the case for the example below, where the measure is is sensitive to 
regions of $\Xspace \times \Yspace$ with \emph{low density} as measured by 
$p(x, y)$. 

\begin{example}[Recall]
  Consider estimating recall $G_\mathrm{rec}$ (row~4 of 
  Table~\ref{tbl:classification-measures}) of a binary classifier.
  By the CLT, the passive estimator for $G_\mathrm{rec}$ is asymptotically 
  normal with variance 
  $G_\mathrm{rec} (1 - G_\mathrm{rec}) / N \epsilon$ where $\epsilon$ denotes 
  the relative frequency of the positive class. 
  Thus we require a labeled sample of size 
  $N \propto G_\mathrm{rec} (1 - G_\mathrm{rec})/w^2 \epsilon$ 
  to estimate $G_\mathrm{rec}$ with precision $w$. 
  This dependence on $\epsilon$ makes passive sampling impractical 
  when $\epsilon \ll 0$---i.e. when the positive class is rare. 
\end{example}

This example is not merely an intellectual curiosity---there are important 
applications where $\epsilon$ is exceedingly small. 
For instance, in record linkage $\epsilon$ scales inversely in the size 
of the databases to be linked \citep{marchant_search_2017}.

%% file: subfiles/ais.tex
\section{An AIS-based framework for label-efficient evaluation}
\label{sec:ais}

\begin{figure}
  \centering
  \includegraphics[scale=1.0]{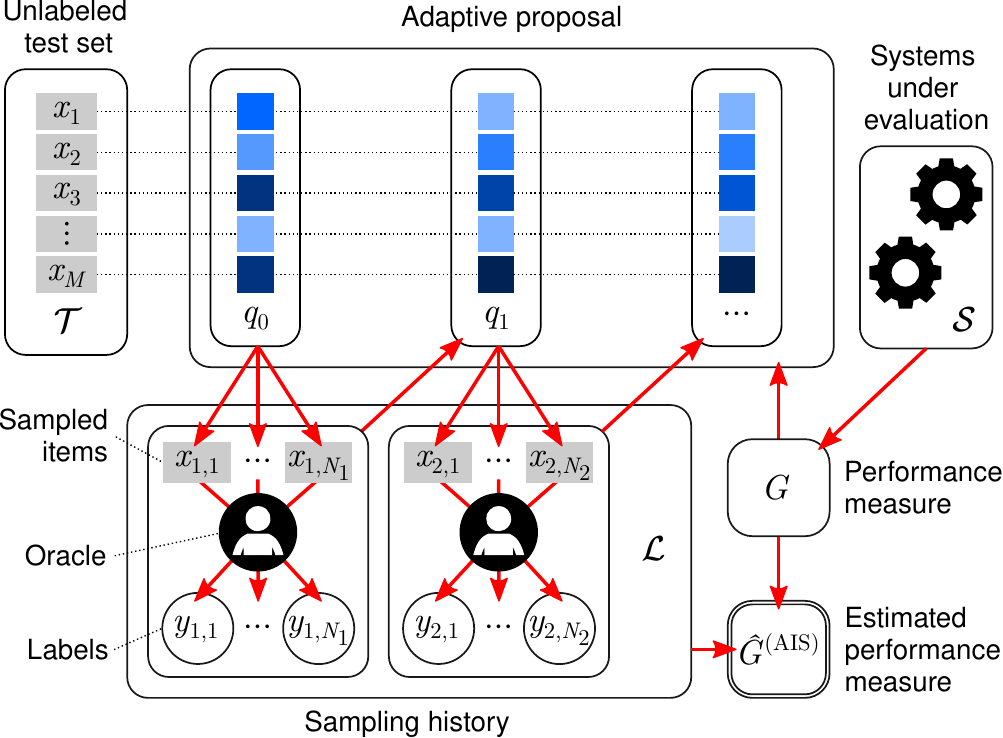}
  \caption{Schematic of the proposed AIS-based evaluation framework.}
  \label{fig:ais-framework}
\end{figure}

\begin{figure}
  \centering
  \begin{minipage}{.6\linewidth}
    \begin{algorithm}[H]
      \begin{algorithmic}
        \Statex \textbf{Input:} generalized measure~$G$; 
        unlabeled test pool~$\testset$; proposal update procedure; sample 
        allocations $N_1, \ldots, N_T$.\vspace{1ex}
        \State Initialize proposal $q_0$ and sample history $\mathcal{L} \gets \emptyset$
        \For{$t = 1$ to $T$}
          \For{$n = 1$ to $N_t$}
            \State $x_{t,n} \sim q_{t - 1}$
            \State $w_{t,n} \gets 
              \frac{p(x_{t,n})}{q_{t - 1}(x_{t,n})}$
            \State $y_{t,n} \sim \operatorname{Oracle}(x_{t,n})$
            \State $\mathcal{L} \gets 
              \mathcal{L} \cup \{(x_{t,n}, y_{t,n}, w_{t,n})\}$
          \EndFor
          \State Update proposal $q_t$ using $\mathcal{L}$
        \EndFor
        \State $\hat{R}_{N}^{\mathrm{AIS}} \gets 
          \frac{1}{N} \sum_{(x,y,w) \in \mathcal{L}} w \, \ell(x,y)$ 
        \State $\hat{G}_{N}^{\mathrm{AIS}} \gets 
          g(\hat{R}_{N}^{\mathrm{AIS}})$
        \State \Return $\hat{G}_{N}^{\mathrm{AIS}}$ and history $\mathcal{L}$
      \end{algorithmic}
      \caption{AIS for generalized measures}
      \label{alg:ais}
    \end{algorithm}
  \end{minipage}
\end{figure}

When passive sampling is inefficient\footnote{%
  Unless otherwise specified, we mean \emph{label efficiency} or \emph{sample 
  efficiency} when we use the term ``efficiency'' without qualification. 
}, as we have seen in the preceding
analysis, substantial improvements can often be achieved through biased 
sampling. 
In this section, we devise a framework for efficiently estimating 
generalized measures that leverages a biased sampling approach called 
\emph{adaptive importance sampling (AIS)}~\citep{bugallo_adaptive_2017}.
AIS estimates an expectation using samples drawn sequentially from a biased 
proposal distribution, that is adapted in stages based on samples from 
previous stages. 
It produces non-i.i.d.\ samples in general, unlike passive sampling and 
static (non-adaptive) \emph{importance sampling (IS)} 
\citep{rubinstein_simulation_2016}.
AIS is a powerful approach because it does not assume an effective 
proposal distribution is known a priori---instead it is learnt from data.
This addresses the main limitation of static IS---that one may be stuck 
using a sub-optimal proposal, which may compromise label efficiency.

There are many variations of AIS which differ in: (i)~the way samples 
are allocated among stages; (ii)~the dependence of the proposal 
on previous stages; (iii)~the types of proposals considered; 
and (iv)~the way samples are weighted within and across stages. 
Our approach is completely flexible with respect to points~(i)--(iii).  
For point~(iv), we use simple importance-weighting as it is amenable 
to asymptotic analysis using martingale theory \citep{portier_asymptotic_2018}. 
A more complex weighting scheme is proposed in~\citep{cornuet_adaptive_2012} 
which may have better stability, however its asymptotic behavior is not 
well understood.\footnote{Consistency was proved for this weighting scheme 
in the limit $T \to \infty$ where $\{N_t\}$ is a monotonically increasing 
sequence \citep{marin_consistency_2019}. 
To our knowledge, a CLT remains elusive.}

Our framework is summarized in Figure~\ref{fig:ais-framework}. 
Given a target performance measure $G$ and an unlabeled test pool 
$\mathcal{T}$, the labeling process proceeds in several stages indexed by 
$t \in \{1, \ldots, T\}$. 
In the $t$-th stage, $N_t \geq 1$ instances are drawn i.i.d.\ from 
$\mathcal{T}$ according to proposal $q_{t - 1}$. 
Labels are obtained for each instance by querying the oracle, and recorded 
with their importance weights in $\mathcal{L}$.
At the end of the $t$-th stage, the proposal is updated for the 
next stage. 
This update may depend on the weighted samples from all previous 
stages as recorded in $\mathcal{L}$. 
At any point during sampling, we estimate $G$ as follows:
\begin{equation}
  \hat{G}_{N}^{\mathrm{AIS}} = g(\hat{\risk}_{N}^{\mathrm{AIS}}) 
  \text{ where }
  \hat{\risk}_{N}^{\mathrm{AIS}} = \frac{1}{N} \sum_{(x, y, w) \in \mathcal{L}} w \, \ell(x, y).
  \label{eqn:G-AIS}
\end{equation}

For generality, we permit the user to specify the sample allocations and 
the procedure for updating the proposals in Figure~\ref{fig:ais-framework}. 
In practice, we recommend allocating a small number of samples to each 
stage, as empirical studies suggest that efficiency improves when the 
proposal is updated more frequently~\citep{portier_asymptotic_2018}. 
However, this must be balanced with practical constraints, as a small 
sample allocation limits the ability to acquire labels in parallel. 
In Section~\ref{sec:adaptive-proposals}, we recommend a practical procedure for 
updating the proposals. 
It approximates the asymptotically-optimal variance-minimizing proposal 
based on an online model of $p(y|x)$. 
We present an example model for $p(y|x)$ in Section~\ref{sec:model-proposal}, 
which can leverage prior information from the systems under evaluation. 
Further practicalities including sample reuse and confidence intervals 
are discussed in Section~\ref{app-sec:practicalities}.

\begin{remark}[Constraint on the proposal]
  \label{rem:constraint}
  In a standard application of AIS for estimating 
  $\E_{X, Y \sim p}[\phi(X, Y)]$, one is free to select any proposal 
  $q_t(x, y)$ with support on the set 
  $\{(x, y) : \|\phi(x,y)\| p(x, y) \neq 0$\}.
  However, we have an additional constraint since we cannot bias sampling 
  from the oracle $p(y|x)$. 
  Thus we consider proposals of the form $q_t(x, y) = q_t(x) p(y|x)$.
\end{remark}

\begin{remark}[Static Importance Sampling]
  \label{rem:static-is}
  Our AIS framework reduces to static importance sampling when 
  $T = 1$ so that all samples are drawn from a single static proposal 
  $q_0$.
\end{remark}

\subsection{Asymptotic analysis}
\label{sec:asymptotic-analysis}
We study the asymptotic behavior of estimates for $G$
produced by the generic framework in Figure~\ref{fig:ais-framework}. 
Since our analysis does not depend on how samples are allocated among 
stages, it is cleaner to identify a sample using a single index 
$j \in \{1, \ldots, N\}$ where $N = \sum_{t = 1}^{T} N_t$ is the total 
number of samples, rather than a pair of indices $(t, n)$.
Concretely, we map each $(t, n)$ to index 
$j = n + \sum_{t' = 1}^{t - 1} N_{t'}$.
With this change, we let $q_{j-1}(x)$ denote the proposal used to 
generate sample $j$. 
It is important to note that this notation conceals the dependence on 
previous samples. 
Thus $q_{j-1}(x)$ should be understood as shorthand for 
$q_{j-1}(x|\mathcal{F}_{j-1})$ where 
$\mathcal{F}_{j} = \sigma((X_1, Y_1), \ldots, (X_j, Y_j))$ 
denotes the filtration. 

Our analysis relies on the fact that $Z_N = N (\hat{R}_N^{\mathrm{AIS}} - R)$ 
is a martingale with respect to $\mathcal{F}_{N}$.
The consistency of $\hat{G}_N^\mathrm{AIS}$ then follows by a strong law of 
large numbers for martingales~\citep{feller_introduction_1971} 
and the continuous mapping theorem. 
A proof is included in Appendix~\ref{app-sec:proof-ais-consistency}.

\begin{theorem}[Consistency]
  \label{thm:ais-consistency}
  Let the support of proposal $q_j(x,y) = q_j(x) p(y|x)$ be a superset of 
  $\{x, y \in \Xspace \times \Yspace : \allowbreak \|\ell(x, y)\| p(x, y) \neq 0 \}$ 
  for all $j \geq 0$ and assume
  \begin{equation}
  \sup_{j \in \mathbb{N}} \E \! \left[\left( \frac{p(X_j)}{q_{j-1}(X_j)} \right)^2 \middle| \mathcal{F}_{j-1} \right] < \infty.
  \label{eqn:finite-var}
  \end{equation}
  Then $\hat{G}_N^{\mathrm{AIS}}$ is strongly consistent for $G$.
\end{theorem}

We also obtain a central limit theorem (CLT), which is useful for assessing 
asymptotic efficiency and computing approximate confidence intervals.
Our proof (see Appendix~\ref{app-sec:proof-ais-clt}) invokes a CLT for 
martingales~\citep{portier_asymptotic_2018} and the multivariate delta method.

\begin{theorem}[CLT]
  \label{thm:ais-clt}
  Suppose
  \begin{equation}
    V_j := {\var} \! \left[\frac{p(X_j)}{q_{j-1}(X_j)} \ell(X_j, Y_j) - \risk 
      \middle| \mathcal{F}_{j-1} \right] 
        \to V_\infty \quad \text{a.s.}, \label{eqn:V-convergence}
  \end{equation}
  where $V_\infty$ is an a.s.\ finite random positive semidefinite matrix, 
  and there exists $\eta > 0$ such that
  \begin{equation}
    \sup_{j \in \mathbb{N}} \, \E \! \left[ \left( \frac{p(X_j)}{q_{j-1}(X_j)} \right)^{2+\eta} 
      \middle| \mathcal{F}_{j-1} \right] < \infty \quad \text{a.s.}
    \label{eqn:asymp-negligible}
  \end{equation}
  Then $\sqrt{N} (\hat{G}_N^{\mathrm{AIS}} - G)$ converges in distribution 
  to a multivariate normal $\mathcal{N}(0, \Sigma)$ with 
  covariance matrix $\Sigma = \Jac g (\risk) V_\infty \Jac g(R)^\intercal$
  where $[\Jac g]_{ij} = \frac{\partial g_i}{\partial R_j}$ is the 
  Jacobian of $g$.
\end{theorem}

\subsection{Variance-minimizing proposals}
\label{sec:adaptive-proposals}
In order to achieve optimal asymptotic efficiency, we would like to use the 
AIS proposal that achieves the minimal asymptotic variance $\Sigma$ 
as defined in Theorem~\ref{thm:ais-clt}. 
We can solve for this optimal proposal using functional calculus, as 
demonstrated in the proposition below.
Note that we must generalize the variance minimization problem for 
vector-valued $G$ since $\Sigma$ becomes a covariance matrix. 
We opt to use the \emph{total variance} (the trace of $\Sigma$) 
since the diagonal elements of $\Sigma$ are directly related to 
statistical efficiency, while the off-diagonal elements measure 
correlations between components of $\hat{G}_N^{\mathrm{AIS}}$ 
that are beyond our control. 
This choice also ensures the functional optimization problem is tractable. 

\begin{proposition}[Asymptotically-optimal proposal]
  \label{prop:asymp-opt-proposal}
  \begin{sloppypar}
  Suppose the Jacobian $\Jac_g(R)$ has full row rank and 
  $\E_{X, Y \sim p} \left[\|\Jac_g(R) \, \ell(X, Y)\|_2^2 \right] > 0$.
  Then the proposal
  \end{sloppypar}
  \begin{equation}
      q^\star(x) = \frac{v(x)}{\int v(x) \, \dd x} 
        \quad \text{with} \quad 
          v(x) = p(x) \sqrt{ \int \|\Jac_g(\risk) \, \ell(x, y)\|_2^2 \, p(y|x) \, \dd y}
    \label{eqn:asym-opt-proposal}
  \end{equation}
  achieves the minimum total asymptotic variance of 
  \begin{equation*}
    \trace \Sigma = \left(\int v(x) \, \dd x\right)^2 - \|\Jac_g(R) \, R\|_2^2. 
  \end{equation*}
  Appendix~\ref{app-sec:zero-tot-var} provides sufficient conditions on 
  $G$ and the oracle which ensure $\trace \Sigma = 0$.
\end{proposition}
%
%

We use the above result to design a practical scheme for adapting a 
proposal for AIS. 
At each stage of the evaluation process, we approximate the 
asymptotically-optimal proposal $q^\star(x)$ using an online model for 
the oracle response $p(y|x)$. 
The model for $p(y|x)$ should be initialized using prior information 
if available (e.g.\ from the systems under evaluation) and updated at the 
end of each stage using labels received from the oracle. 
However, we cannot simply estimate $q^\star(x)$ by plugging in estimates 
of $p(y|x)$ directly, as the resulting proposal may not satisfy the 
conditions of Theorems~\ref{thm:ais-consistency} and~\ref{thm:ais-clt}. 
Below we provide estimators for $q^\star(x)$ which do satisfy the conditions, 
and provide sufficient conditions for achieving asymptotic optimality.

\begin{proposition}
  \label{prop:valid-proposal-est}
  If the oracle is \emph{stochastic}, let $\hat{p}_t(y|x)$ be an estimate 
  for $p(y|x)$ whose support includes the support of $p(y|x)$ for all 
  stages $t \geq 0$, and assume 
  $\hat{p}_t(y|x) \overset{\mathrm{a.s.}}{\to} \hat{p}_\infty(y|x)$ 
  pointwise in $x$. 
  Alternatively, if the oracle is \emph{deterministic}, let $\pi_t(y|x)$ 
  be a posterior distribution for the response $y(x)$ whose support 
  includes $y(x)$ for all $t \geq 0$, and assume 
  $\pi_t(y|x) \overset{\mathrm{a.s.}}{\to} \pi_\infty(y|x)$ pointwise 
  in $x$.
  Let $\epsilon_t$ be a positive bounded sequence and $\hat{R}_t$ be 
  an estimator for $R$ which converges a.s.\ to $\hat{R}_\infty$. 
  Assume $\Xspace$ is finite (e.g.\ a pool of test data) and 
  $\|\Jac_g(\hat{R}_t)\|_2 \leq K < \infty$ for all $t \geq 0$. 
  Then the proposals 
  \begin{equation*}
     q_t(x) \propto \begin{cases} 
       p(x) \int \max\{ \| \Jac_g(\hat{R}_t) \, \ell(x, y) \|_2, 
        \epsilon_t \ind{\|\ell(x, y)\| \neq 0} \} 
          \pi_t(y|x) \, \dd y, &\text{for a deterministic oracle,} \\
      p(x) \left[ \int \max\{ \| \Jac_g(\hat{R}_t) \, \ell(x, y) \|_2^2, 
        \epsilon_t \ind{\|\ell(x, y)\| \neq 0} \} 
          \hat{p}_t(y|x) \, \dd y \right]^{\frac{1}{2}}, &\text{for a stochastic oracle,} 
    \end{cases}
  \end{equation*}
  satisfy the conditions of Theorems~\ref{thm:ais-consistency} 
  and~\ref{thm:ais-clt}. 
  If in addition $\hat{R}_\infty = R$ and $\hat{p}_\infty(y|x) = p(y|x)$ 
  (alternatively $\pi_\infty(y|x) = \ind{y = y(x)}$) and 
  $\epsilon_t \downarrow 0$, then the proposals approach 
  asymptotic optimality.
\end{proposition}

\subsection{Practicalities}
\label{app-sec:practicalities}

We briefly discuss solutions to two issues that may arise in practical 
settings: sample reuse and approximate confidence regions.

\subsubsection{Sample reuse}
Suppose our framework is used to estimate a generalized 
measure $G_1$. 
If the joint distribution $p(x,y)$ associated with the prediction 
problem has not changed, it may be desirable to \emph{reuse} the 
weighted samples $\mathcal{L}$ to estimate a different generalized 
measure $G_2$. 
This is possible so long as the sequence of proposals used to estimate 
$G_1$ have the required support for $G_2$. 
More precisely, the support of $q_j(x,y)$ must include
$\{x, y \in \Xspace \times \Yspace : \|\ell(x, y)\| p(x, y) \neq 0 \}$ 
for the loss functions associated with $G_1$ \emph{and} $G_2$.

If one anticipates sample reuse, the proposals can be made 
less specialized to a particular measure by mixing with the marginal 
distribution $p(x)$, i.e.\ $q(x) \to (1 - \delta) q(x) + \delta p(x)$
where $\delta \in (0, 1]$ is a hyperparameter that controls the 
degree of specialization.

\subsubsection{Approximate confidence regions}
\label{app-sec:confidence-region}
When publishing performance estimates, it may be desirable to quantify  
statistical uncertainty.
An asymptotic $100(1-\alpha)\%$ confidence region for a generalized measure 
$G$ is given by the ellipsoid
\begin{equation*}
  \left\{G^\star \in \reals^k : (G^\star - \hat{G})^\intercal \hat{\Sigma}^{-1} (G^\star - \hat{G}) \leq 
    \frac{(N - 1)k}{N(N - k)} F_{\alpha,k,N-k} \right\},
\end{equation*}
where $\hat{G}$ is the sample mean, $\hat{\Sigma}$ is the sample 
covariance matrix, and $F_{\alpha,d_1,d_2}$ is the critical value of the 
$F$ distribution with $d_1, d_2$ degrees of freedom at significance level 
$\alpha$.
This region can be approximated using the estimator for $G$ in 
\eqref{eqn:G-AIS} and the following estimator for $\Sigma$:
\begin{equation*}
    \hat{\Sigma}^{\mathrm{AIS}} = \Jac_g (\hat{R}^{\mathrm{AIS}}) \bigg( \frac{1}{N} \sum_{j = 1}^{N} 
        \frac{p(x_j)^2 \ell(x_j, y_j) \ell(x_j, y_j)^\intercal}{q_N(x_j) q_{j-1}(x_j)} \\
       - \hat{R}^{\mathrm{AIS}} \hat{R}^{\mathrm{AIS}\intercal} \bigg) \Jac_g (\hat{R}^{\mathrm{AIS}})^\intercal.
\end{equation*}
This is obtained from the expression for $\Sigma$ in Theorem~\ref{thm:ais-clt},
by plugging in AIS estimators for the variance and $R$, and approximating 
$q_\infty(x)$ by the most recent proposal $q_N(x)$.

\section{A Dirichlet-tree model for the oracle response}
\label{sec:model-proposal}

In the previous section, we introduced a scheme for updating the AIS proposal 
which relies on an online model of the oracle response. 
Since there are many conceivable choices for the model, we left it 
unspecified for full generality.
In this section, we propose a particular model that is suited for evaluating 
\emph{classifiers} when the response from the oracle is \emph{deterministic} 
(see Remark~\ref{rem:oracle}).
Concretely, we make the assumption that the label space 
$\Yspace = \{1, \ldots, C\}$ is a finite set and $p(y|x)$ is a point mass at 
$y(x)$ for all $x \in \testset$. 
An extension of this section for stochastic oracles (where $p(y|x)$ is not 
necessarily a point mass) is presented in 
Appendix~\ref{app-sec:oracle-est-stoc}. 

Since we would like to leverage prior information (e.g.\ classifier scores) 
from the system(s) under evaluation and perform regular updates as labels 
are received from the oracle, we opt to use a Bayesian model.
Another design consideration is label efficiency. 
Since labels are scarce and the test pool $\testset$ may be huge, we want 
to design a model that allows for sharing of statistical strength between 
``similar'' instances in $\testset$.
To this end, we propose a model that incorporates a hierarchical partition 
of $\testset$, where instances assigned to hierarchically neighboring blocks 
are assumed to elicit a similar oracle response.\footnote{
  This is in contrast to models used in related work 
  \citep{bennett_online_2010,marchant_search_2017} which assume the 
  oracle response is independent across blocks of a non-hierarchical 
  partition.
}
Various unsupervised methods may be used to learn a hierarchical partition, 
including hierarchical agglomerative\slash divisive clustering 
\citep{reddy_survey_2010}, $k$-d trees \citep{bentley_multidimensional_1975}, 
and stratification based on classifier scores (see 
Appendix~\ref{app-sec:unsupervised-partition} for a brief review).

\subsection{Generative process}
We assume the global oracle response $\theta$ (averaged over all instances) 
is generated according to a Dirichlet distribution, viz.
\begin{equation*}
  \theta | \alpha \sim \operatorname{Dirichlet} \! \left(\alpha\right),
\end{equation*}
where $\alpha = [\alpha_1, \ldots, \alpha_C] \in \reals_+^C$ are concentration 
hyperparameters.
The label $y_i$ for each instance $i \in \{1, \ldots, M\}$ (indexing 
$\testset$) is then assumed to be generated i.i.d.\ according to $\theta$:
\begin{align*}
  y_i | \theta & \overset{\mathrm{iid.}}{\sim} \operatorname{Categorical} \! \left(\theta\right), 
    & i \in 1, \ldots, M.
\end{align*} 

We assume a hierarchical partition of the test pool $\testset$ is given. 
The partition can be represented as a tree $T$, where the leaf nodes of $T$ 
correspond to the finest partition of $\testset$ into disjoint blocks 
$\{ \testset_k \}_{k = 1}^{K}$ such that $\testset = \bigcup_{k = 1}^{K} 
\testset_k$. 
We assume each instance $i$ is assigned to one of the blocks (leaf nodes)
$k_i \in \{1, \ldots, K\}$ according to a distribution $\psi_y$ with a 
Dirichlet-tree prior \citep{dennis_bayesian_1996, minka_dirichlet-tree_1999}:
\begin{align*}
  \psi_{y} | \beta_y, T & \overset{\mathrm{ind.}}{\sim} \operatorname{DirichletTree}\! \left(\beta_y; T\right), 
    & y \in \Yspace, \\
  k_i | y_i, \psi_{y_i} & \overset{\mathrm{ind.}}{\sim} \operatorname{Categorical} \! \left(\psi_{y_i}\right),
    & i \in 1, \ldots, M.
\end{align*}
The Dirichlet-tree distribution is a generalization of the Dirichlet 
distribution, which allows for more flexible dependencies between the 
categories (blocks in this case). 
Categories that are hierarchically nearby based on the tree structure $T$ 
tend to be correlated. 
The Dirichlet concentration hyperparameters $\beta_y$ associated with the 
internal nodes also control the correlation structure.

\subsection{Inference} 
For a deterministic oracle, the response $y_i$ for instance $i$ is either 
observed (previously labeled) or unobserved (yet to be labeled). 
It is important to model the observation process in case it influences the 
values of inferred parameters. 
To this end, we let $\vec{o}_t = (o_{t,1}, \ldots, o_{t,M})$ be observation 
indicators for the labels $\vec{y} = (y_1, \ldots, y_M)$ at the end of stage 
$t$ of the evaluation process (see Algorithm~\ref{alg:ais}). 
We initialize $\vec{o}_{0} = 0$ and define $\vec{o}_t$ in the obvious way: 
$o_{t,i}$ is 1 if the label for instance $i$ has been observed by the end 
of stage $t$ and 0 otherwise. 
From Algorithm~\ref{alg:ais}, we have that the $n$-th instance selected in 
stage $t$ depends on the labels of the previously observed instances 
$\vec{y}_{(\vec{o}_{t-1})}$ and the block assignments 
$\vec{k} = (k_1, \ldots, k_M)$:
\begin{equation*}
i_{t,n} | \vec{o}_{t-1}, \vec{y}_{(\vec{o}_{t-1})}, \vec{k} \sim q_{t-1}(\vec{y}_{(\vec{o}_t)}, \vec{k}).
\end{equation*}

Our goal is to infer the unobserved labels (and hence the oracle response) 
at each stage $t$ of the evaluation process. 
We assume the block assignments $\vec{k} = (k_1, \ldots, k_M)$ are fully 
observed. 
Since the observation indicators are \emph{independent} of the unobserved 
labels conditional on the observed labels, our model satisfies ignorability 
\citep{jaeger_ignorability_2005}. 
This means we can ignore the observation process when conducting inference. 
Since we do not require a full posterior distribution over all parameters, 
it is sufficient to conduct inference using the expectation-maximization 
algorithm. 
This yields a distribution over the unobserved label for each instance and 
point estimates for the other parameters ($\psi_y$ and $\theta$). 
Full details are provided in Appendix~\ref{app-sec:oracle-est-det}.

\subsection{Asymptotic optimality}

Since the Dirichlet-tree model described in this section is consistent for 
the true oracle (deterministic) response, it can be combined with the proposal 
updates described in Proposition~\ref{prop:valid-proposal-est} to yield 
an asymptotically-optimal AIS algorithm. 
This result is made precise in the following proposition, which is proved 
in Appendix~\ref{app-sec:proof-valid-proposal-det}.

\begin{proposition}
  \label{prop:valid-proposal-det}
  Consider an instantiation of our framework under a deterministic 
  oracle where:
  \begin{itemize}
    \item the oracle response is estimated online using the Dirichlet-tree 
    model described in this section via the EM~algorithm;
    \item the proposals are adapted using the estimator defined in 
    Proposition~\ref{prop:valid-proposal-est} with 
    \item $\epsilon_t = \epsilon_0 (1 - \frac{1}{M} \sum_{i = 1}^{M} o_{t,i})$ 
    for some user-specified $\epsilon_0 > 0$.
  \end{itemize}
  Then Theorems~\ref{thm:ais-consistency} and~\ref{thm:ais-clt} hold and 
  the estimator is asymptotically-optimal.
\end{proposition}

%% file: subfiles/experiments.tex
\section{Experimental study}
\label{sec:experiments}

We conduct experiments to assess the label efficiency of our proposed 
framework\footnote{Using the procedure described in 
Section~\ref{sec:adaptive-proposals} to adapt the AIS proposal, together 
with the online model for the oracle response presented in 
Section~\ref{sec:model-proposal}.} for a variety of evaluation tasks. 
The tasks vary in terms of the degree of class imbalance, the quality of 
predictions\slash scores from the classifier under evaluation, the size of 
the test pool, and the target performance measure. 
Where possible, we compare our framework (denoted \textsf{Ours}) with the 
following baselines:
\begin{itemize}[itemsep=0pt,topsep=0pt]
  \item \textsf{Passive}: passive sampling as specified in 
  Definition~\ref{def:passive}.
  \item \textsf{IS}: static importance sampling as described in 
  Remark~\ref{rem:static-is}.
  We approximate the asymptotically-optimal proposal as in  
  Proposition~\ref{prop:valid-proposal-est} using estimates of 
  $p(y|x)$ derived from classifier scores. 
  \item \textsf{Stratified}: an online variant of stratified sampling 
  with proportional allocation, as used in \citep{druck_toward_2011}.
  Items are sampled one-at-a-time in proportion to the size of the 
  allocated stratum. 
  \item \textsf{OASIS}: a stratified AIS method for estimating F scores 
  \citep{marchant_search_2017}. 
\end{itemize}

\subsection{Evaluation tasks}
We prepare classifiers and test pools for evaluation using publicly-available 
datasets from various domains, as summarized in Table~\ref{tbl:data-sets}. 
\texttt{amzn-googl}, \texttt{dblp-acm}, \texttt{abt-buy} 
\citep{kopcke_evaluation_2010} and \texttt{restaurant} \citep{riddle_datasets} 
are benchmark datasets for entity resolution. 
They contain records from two sources and the goal is to 
predict whether pairs of records refer to the same entity or not. 
\texttt{safedriver} contains anonymized records from a car insurance 
company, and the task is to predict drivers who are likely to make a 
claim \citep{porto_kaggle}. 
\texttt{creditcard} relates to fraud detection for online credit 
card transactions \citep{pozzolo_calibrating_2015}. 
\texttt{tweets100k} has been applied to tweet sentiment analysis 
\citep{mozafari_scaling_2014}.

For \texttt{amzn-goog}, \texttt{dblp-acm}, \texttt{abt-buy}, 
\texttt{restaurant} and \texttt{tweets100k} we use the same classifiers 
and test pools as in \citep{marchant_search_2017}. 
For \texttt{safedriver} and \texttt{creditcard} we prepare our own by 
randomly splitting the data into train\slash test with a 70:30 ratio, and 
training classifiers using supervised learning. 
In scenarios where labeled data is scarce, semi-supervised or unsupervised 
methods might be used instead---the choice of learning paradigm has no 
bearing on evaluation. 
We consider three target performance measures---F1 score, accuracy and 
precision-recall curves---as separate evaluation tasks.

\begin{table*}
  \caption{Summary of test pools and classifiers under evaluation. 
  The imbalance ratio is the number of positive class instances 
  divided by the number of negative class instances. 
  The true F1~score is assumed unknown.}
  \label{tbl:data-sets}
  \centering
    \begin{tabular}{llrclc}
      \toprule
      Source                                               & Domain             
        & Size    & Imbalance ratio & Classifier type      & F1 score \\
      \midrule
      \texttt{abt-buy} \citep{kopcke_evaluation_2010}      & Entity resolution  
        &  53,753 & 1075       & Linear SVM          & 0.595 \\
      \texttt{amzn-goog} \citep{kopcke_evaluation_2010}    & Entity resolution  
        & 676,267 & 3381       & Linear SVM          & 0.282 \\
      \texttt{dblp-acm} \citep{kopcke_evaluation_2010}     & Entity resolution  
        &  53,946 & 2697       & Linear SVM          & 0.947 \\
      \texttt{restaurant} \citep{riddle_datasets}          & Entity resolution  
        & 149,747 & 3328       & Linear SVM          & 0.899 \\
      \texttt{safedriver} \citep{porto_kaggle}             & Risk assessment    
        & 178,564 & 26.56      & XGBoost \citep{chen_xgboost_2016} & 0.100 \\
      \texttt{creditcard} \citep{pozzolo_calibrating_2015} & Fraud detection    
        &  85,443 & 580.2      & Logistic Regression & 0.728  \\
      \texttt{tweets100k} \citep{mozafari_scaling_2014}    & Sentiment analysis 
        &  20,000 & 0.990      & Linear SVM          & 0.770 \\
      \bottomrule
    \end{tabular}
\end{table*}

\subsection{Setup}
\paragraph{Oracle.}
We simulate an oracle using labels included with each dataset. 
Since the datasets only contain a single label for each instance, we 
assume a deterministic oracle. 
Thus, when computing the consumed label budget, we only count the 
\emph{first} query to the oracle for an instance as a consumed label.
If an instance is selected for labeling again, we reuse the label 
from the first query. 

\paragraph{Partitioning.}
\textsf{Ours}, \textsf{Stratified} and \textsf{OASIS} assume the test 
pool is partitioned so the oracle response within each block is ideally 
uniform.
We construct partitions by binning instances according to their classifier 
scores. 
The bin edges are determined using the \emph{cumulative square-root frequency 
(CSF) method} \citep{dalenius_minimum_1959}, which is widely used for 
stratified sampling.
We set the number of bins to $K = 2^8$. 
Since \textsf{Ours} is capable of exploiting partitions with hierarchical 
structure, we fill in post-hoc structure by associating the CSF bins 
with the leaf nodes of an appropriately-size tree in breadth-first order. 
We consider two trees: a tree of depth~1 with branching factor $K$ (denoted 
\textsf{Ours-1}, equivalent to a non-hierarchical partition) 
and a tree of depth~8 with branching factor~2 (denoted \textsf{Ours-8}). 

\paragraph{Hyperparameters.}
We leverage prior information from the classifiers under evaluation to 
set hyperparameters. 
Wherever a prior estimate of $p(y|x)$ is required, we use the classifier score 
$s(y|x)$, applying the softmax function to non-probabilistic scores.
For the Dirichlet-tree model we set hyperparameters as follows:
\begin{equation*}
  \alpha_y = 1 + \sum_{k = 1}^{K} s(y|k), \quad 
  \beta_{y \nu} = \operatorname{depth}(\nu)^2 + \sum_{k = 1}^{K} s(y|k) \delta_\nu(k)
\end{equation*}
where $s(y|k) = \frac{1}{|\testset_k|} \sum_{x_i \in \testset_k} s(y|x_i)$ 
is the mean score over instances in the $k$-th block, 
$\nu$ denotes a non-root node of the tree $T$, 
$\operatorname{depth}(\nu)$ denotes the depth of node $\nu$ in $T$, 
and $\delta_{\nu}(k)$ is an indicator equal to 1 if node $\nu$ is traversed 
to reach leaf node $k$ and 0 otherwise.

\paragraph{Repeats.}
Since the evaluation process is randomized, we repeated each experiment 
1000 times, computing and reporting the mean behavior with bootstrap 
confidence intervals. 

\subsection{Results}

We provide a summary of the results here, with a particular focus on the 
results for F1 score, which is supported by all baselines.
Detailed results for accuracy and precision-recall curves are included 
in Appendix~\ref{app-sec:extra-plots}.

\paragraph{F1 score.}

To assess convergence of the estimated F1~score, we plot the mean-squared 
error (MSE) as a function of the consumed label budget for all datasets 
and methods.
A plot of this kind is included for \texttt{abt-buy} in 
Figure~\ref{fig:convergence-abt-buy}. 
It shows that \texttt{Ours} is significantly more efficient than the 
baseline methods in this case, achieving a lower MSE for all label budgets. 
The \textsf{Passive} and \textsf{Stratified} methods perform significantly 
worse, achieving an MSE at least one order of magnitude greater than 
the biased sampling methods. 
Figure~\ref{fig:convergence-abt-buy} also plots the convergence of the 
proposal in terms of the mean KL divergence from the 
asymptotically-optimal proposal. 
The results here are in line with expectations: convergence of the F1 score 
is more rapid when the proposal is closer to asymptotic optimality.

\begin{figure}
  \centering
  \includegraphics[scale=1.0]{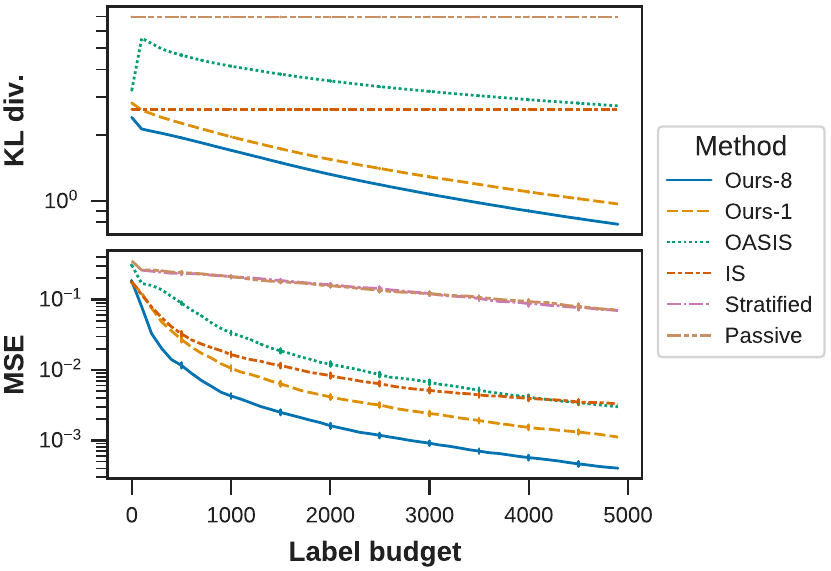}
  \caption{Convergence for \texttt{abt-buy} over 1000 repeats. 
  The upper panel plots the KL divergence from the proposal 
  to the asymptotically-optimal one. 
  The lower panel plots the MSE of the estimated F1 score. 
  95\% bootstrap confidence intervals are included.}
  \label{fig:convergence-abt-buy}
\end{figure}

\begin{figure}
  \centering
  \includegraphics[scale=1.0]{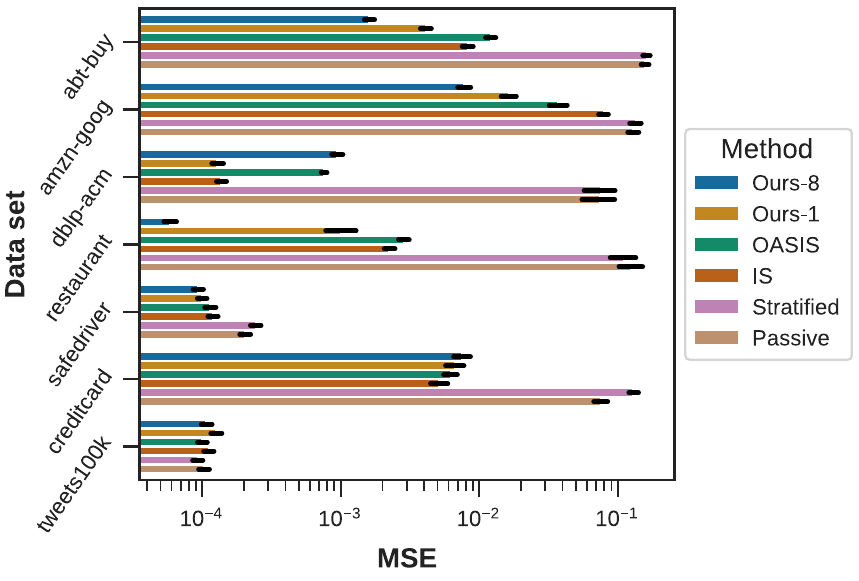}
  \caption{MSE of the estimated F1 score after 2000 label queries over 
  1000 repeats.
  The order of the bars (from top to bottom) for each dataset matches 
  the order in the legend. 
  95\% bootstrap confidence intervals are shown in black.}
  \label{fig:convergence-summary}
\end{figure}

Figure~\ref{fig:convergence-summary} summarizes the convergence plots for 
the six other data sets (included in Appendix~\ref{app-sec:extra-plots}), 
by plotting the MSE of the estimated F1 score after 2000 labels are consumed.
It shows that \textsf{Ours} achieves best or equal-best MSE on all 
but one of the datasets (\texttt{dblp-acm}) within the 95\% confidence 
bands. 
We find the adaptive methods \textsf{Ours} and \textsf{OASIS} perform 
similarly to \texttt{IS} when the prior estimates for $p(y|x)$ are 
accurate---i.e.\ there is less to gain by adapting. 
Of the two adaptive methods, \textsf{Ours} generally converges more 
rapidly than \textsf{OASIS}, which might be expected since our 
procedure for adapting the proposal is asymptotically-optimal.
The hierarchical variant of our model \textsf{Ours-8} tends to 
outperform the non-hierarchical variant \textsf{Ours-1}, which 
we expect when $p(y|x)$ varies smoothly across neighboring blocks. 
Finally, we observe that \textsf{Passive} and \textsf{Stratified} are 
competitive when the class imbalance is less severe. 
This agrees with our analysis in Section~\ref{sec:limitations-mc}. 
We note that similar trends when targeting accuracy in place of F1~score, 
as presented in Appendix~\ref{app-sec:extra-plots}.

\paragraph{PR curves.}
We estimate PR curves on a uniformly-spaced grid of classifier 
thresholds $\tau_1 < \cdots < \tau_L$, where $\tau_1$ is the minimum 
classifier score, $\tau_L$ is the maximum classifier score, and 
$L = 2^{10}$ (see Example~\ref{ex:pr-curve}).
We also use the uniform grid to partition the test pool (in place of the CSF 
method), associating each block of the partition with four neighboring bins 
on the grid to yield $K = 2^8$ blocks. 
Figure~\ref{fig:pr-curves} illustrates the vast improvement of the 
biased sampling methods (\textsf{Ours-8} and \textsf{IS}) over 
\textsf{Passive} for this evaluation task.
The estimated PR curves shown for \textsf{Ours-8} and \texttt{IS} vary 
minimally about the true PR curve, and are reliable for selecting 
an operating threshold. 
The same cannot be said for the curves produced by \textsf{Passive}, 
which exhibit such high variance that they are essentially useless.

\begin{figure}
  \centering
  \includegraphics[scale=1.0]{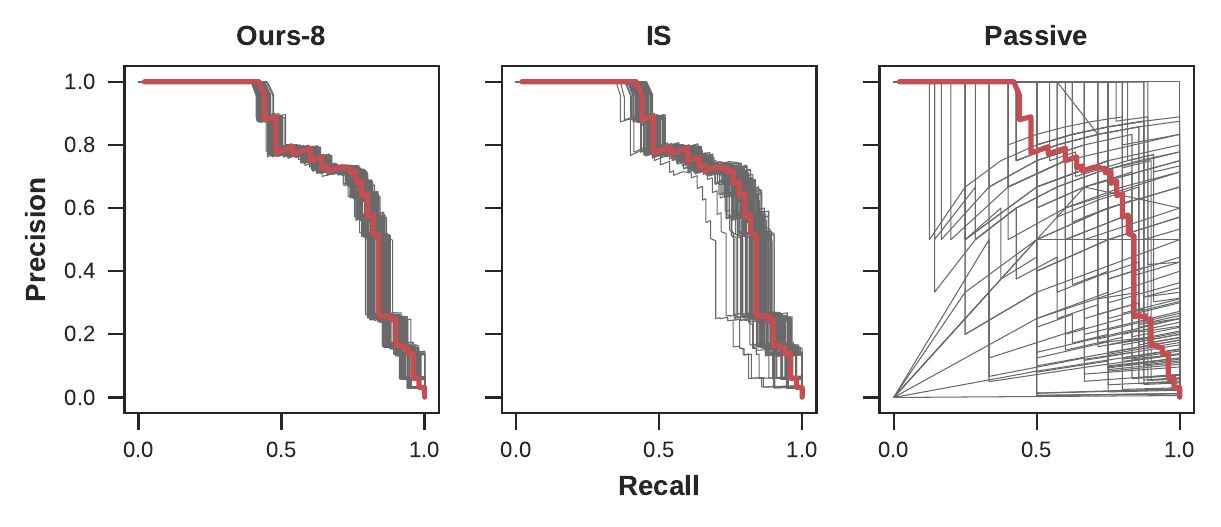}
  \caption{A sample of 100 estimated precision-recall (PR) curves (in dark 
  gray) for three evaluation methods: \textsf{Ours-8}, \textsf{IS} and 
  \textsf{Passive}. 
  The PR curves are estimated for \texttt{abt-buy} using a label budget of 
  5000.
  The thick red curve is the true PR curve (assuming all labels are known).}
  \label{fig:pr-curves}
\end{figure}

%% file: subfiles/related.tex
\section{Related work}
\label{sec:related}

Existing approaches to label-efficient evaluation in a machine learning
context largely fall into three categories: model-based 
\citep{welinder_lazy_2013}, stratified sampling \citep{bennett_online_2010, 
druck_toward_2011} and importance sampling 
\citep{sawade_active_2010, marchant_search_2017}. 
The model-based approach in~\citep{welinder_lazy_2013} estimates 
precision-recall curves for binary classifiers. 
However, it uses inefficient passive sampling to select instances to label, 
and makes strong assumptions about the distribution of scores and labels 
which can result in biased estimates. 
Stratified sampling has been used to estimate scalar performance measures 
such as precision, accuracy and F1 score. 
Existing approaches~\citep{bennett_online_2010, druck_toward_2011} 
bias the sampling of instances from strata (akin to blocks) using a heuristic 
generalization of the optimal allocation 
principle~\citep{cochran_sampling_1977}. 
However, stratified sampling is considered to be a less effective 
variance reduction method compared to importance 
sampling~\citep{rubinstein_simulation_2016}, and it does not naturally 
support stochastic oracles.

Static importance sampling~\citep{sawade_active_2010} and stratified 
adaptive importance sampling~\citep{marchant_search_2017} have 
been used for online evaluation, and are most similar to our approach. 
However \citep{marchant_search_2017} only supports the estimation of 
F1 score, and \citep{sawade_active_2010} only supports the estimation of 
scalar generalized risks\footnote{These can be viewed as a sub-family 
of generalized measures with the following correspondence
$\ell(x, y) = w(x, y, f(x)) [\ell(f(x), y), 1]^\intercal$, 
and $g(\risk) = \risk_1/\risk_2$.}. 
Both of these methods attempt to approximate the asymptotically-optimal 
variance-minimizing proposal, however the approximation used 
in~\citep{sawade_active_2010} is non-adaptive and is not optimized for 
deterministic oracles, while the approximation used 
in~\citep{marchant_search_2017} is adaptive but less accurate due to 
the stratified design. 

Novel evaluation methods are also studied in the information retrieval 
(IR) community (see survey \citep{kanoulas_short_2016}). 
Some tasks in the IR setting can be cast as prediction problems by 
treating query-document pairs as inputs and relevance judgments as 
outputs. 
Early approaches used relevance scores from the IR system to manage 
the abundance of irrelevant documents in an ad hoc manner 
\citep{cormack_efficient_1998}. 
Recent approaches \citep{schnabel_unbiased_2016, li_active_2017} are 
based on a statistical formulation similar to ours, however they are 
specialized to IR systems. 
Within the IR community, stratified sampling and cluster sampling have 
also been used to efficiently evaluate knowledge 
graphs~\cite{gao_efficient_2019}.

Adaptive importance sampling (AIS) is studied more generally in the context 
of Monte Carlo integration (see review by \citealp{bugallo_adaptive_2017}). 
Most methods are inappropriate for our application, as they assume 
a continuous space without constraints on the proposal (see 
Remark~\ref{rem:constraint}). 
\citet{oh_adaptive_1992} introduced the idea of adapting the proposal 
over multiple stages using samples from the previous stages.  
\citet{cappe_adaptive_2008} devise a general framework using independent 
mixtures as proposals. 
The method of \citet{cornuet_adaptive_2012} continually re-weights 
all past samples, however it is more computationally demanding and 
less amenable to analysis since it breaks the martingale property. 
\citet{portier_asymptotic_2018} analyze parametric AIS in the large 
sample limit. 
This improves upon earlier work which assumed the number of stages 
goes to infinity \citep{douc_limit_2008} or the sample size at 
each stage is monotonically increasing \citep{marin_consistency_2019}.

Finally, we note that label-efficient evaluation may be viewed as a 
counterpart to \emph{active learning}, as both are concerned with 
reducing labeling requirements. 
There is a large body of literature concerning active learning---we 
refer the reader to surveys~\citep{settles_active_2009, gilyazev_active_2018}.
However whereas active learning aims to find a model with low bounded 
risk using actively-sampled training data, active evaluation aims to 
estimate risk using actively-sampled test data for \emph{any model}.

%% file: subfiles/conclusion.tex
\section{Conclusion}
\label{sec:conclusion}
We have proposed a framework for online supervised evaluation, 
which aims to minimize labeling efforts required to achieve precise, 
asymptotically-unbiased performance estimates. 
Our framework is based on adaptive importance sampling, with 
variance-minimizing proposals that are refined adaptively based 
on an online model of $p(y|x)$. 
Under verifiable conditions on the chosen performance measure and 
the model, we proved strong consistency (asymptotic 
unbiasedness) of the resulting performance estimates and a central 
limit theorem. 
We instantiated our framework to evaluate classifiers using
deterministic or stochastic human annotators. 
Our approach based on a hierarchical Dirichlet-tree model, 
achieves superior or competitive performance on all but one of  
seven test cases.

%% file: subfiles/appendix.tex
\break
\appendix

\section{Proof of Theorem~\ref{thm:ais-consistency}}
\label{app-sec:proof-ais-consistency}
First we prove that $\hat{\risk}_N^{\mathrm{AIS}} \overset{\mathrm{a.s.}}{\to} \risk$ 
using a strong law of large numbers (SLLN) for martingales 
\citep[p.~243]{feller_introduction_1971}. 
Consider the martingale 
\begin{equation*}
Z_N = N (\hat{R}_N^{\mathrm{AIS}} - R) 
  = \sum_{j = 1}^{N} \left\{ \frac{p(X_j)}{q_{j-1}(X_j)} \ell(X_j, Y_j) 
    - \risk \right\}
\end{equation*}
and let 
$\delta_{j,i} = \frac{p(X_j)}{q_{j-1}(X_j)} \ell_i(X_j, Y_j) - \risk_i$
denote the $i$-th component of the $j$-th contribution to $Z_N$.
Since $X_j$ is drawn from $q_{j-1}(x)$ and 
$q_{j-1}(x) > 0$ wherever $p(x) \|\ell(x, y)\| \neq 0$, it follows 
that $\E[\delta_{j,i}|\mathcal{F}_{j-1}] = 0$. 
In addition, we have 
\begin{align*}
  \sum_{j = 1}^{\infty} \frac{\E\! \left[\delta_{j,i}^2\right]}{j^2} 
    & = \sum_{j = 1}^{\infty} \frac{1}{j^2} \left\{ \E\! \left[ 
          \frac{p(X_j)^2}{q_{j-1}(X_j)^2} \ell_i(X_j, Y_j)^2 
            \right] + R_i^2 \right\} 
    \leq \sum_{j = 1}^{\infty} \frac{U^2 C}{j^2} 
    < \infty,
\end{align*}
where the inequality follows from the boundedness of $\ell(x,y)$ 
and \eqref{eqn:finite-var}.
Thus the conditions of the SLLN are satisfied and we have 
$\frac{1}{N} \sum_{j = 1}^{N} \delta_{i,j} \overset{\mathrm{a.s.}}{\to} 0 
\implies \hat{\risk}_N^{\mathrm{AIS}} \overset{\mathrm{a.s.}}{\to} \risk$.
Now the continuous mapping theorem states that 
\begin{equation*}
\hat{\risk}_N^{\mathrm{AIS}} \overset{\mathrm{a.s.}}{\to} \risk
  \quad \implies \quad 
  g(\hat{\risk}_N^{\mathrm{AIS}}) \overset{\mathrm{a.s.}}{\to} g(\risk),
\end{equation*}
provided $R$ is not in the set of discontinuity points of $g$.
This condition is satisfied by assumption.

\section{Proof of Theorem~\ref{thm:ais-clt}}
\label{app-sec:proof-ais-clt}
The CLT of \citet{portier_asymptotic_2018} implies that 
$\sqrt{N}(\hat{R}^\mathrm{AIS} - R) \Rightarrow \mathcal{N}(0, V_\infty)$, 
provided two conditions hold. 
The first condition of their CLT holds by assumption \eqref{eqn:V-convergence} 
and the second condition holds by the boundedness of the loss function 
and \eqref{eqn:asymp-negligible}.
The multivariate delta method \citep{vaart_delta_1998} then implies that 
$\sqrt{N}(g(\hat{R}^\mathrm{AIS}) - g(R)) \Rightarrow 
\mathcal{N}(0, \Jac g(\risk) V_\infty \Jac g(\risk)^\intercal)$, since 
$g$ is assumed to be differentiable at $R$ in Definition~\ref{def:gen-measure}.

\section{Proof of Proposition~\ref{prop:asymp-opt-proposal}}
We want to find the proposal $q$ that minimizes the total asymptotic 
variance
\begin{equation*}
  \trace \Sigma = \underset{X, Y \sim p}{\E} \left[ 
      \frac{p(X) \| \Jac_g(R) \, \ell(X,Y) \|_2^2}{q(X)} 
      \right] - \|\Jac_g(R) \, R\|_2^2. 
\end{equation*} 
We can express this as a functional optimization problem:
\begin{equation}
  \begin{aligned}
    \min_{q}      \quad & \int \frac{c(x)}{q(x)} \dd x \\
    \textrm{s.t.} \quad & \int q(x) \, \dd x = 1,
  \end{aligned}
  \label{app-eqn:asym-opt-proposal-prob}
\end{equation}
where 
$c(x) = p(x)^2 \int \| \Jac_g(R) \, \ell(x,y) \|_2^2 \, p(y|x) \dd y$.

Using the method of Lagrange multipliers, \citet{sawade_active_2010} show 
that the solution to \eqref{app-eqn:asym-opt-proposal-prob} is 
$q^\star(x) \propto \sqrt{c(x)}$.
This yields the required result.

\section{Conditions for achieving zero total asymptotic variance}
\label{app-sec:zero-tot-var}

In typical applications of IS, the asymptotically-optimal proposal can 
achieve zero variance. 
This is not guaranteed in our application, since we do not have complete 
freedom in selecting the proposal (see Remark~\ref{rem:constraint}). 
Below we provide sufficient conditions under which the total 
asymptotic variance of $\hat{G}_N^{\mathrm{AIS}}$ can be reduced to zero.

\begin{proposition}
  \label{app-prop:zero-tot-var}
  Suppose the oracle is deterministic (i.e.\ $p(y|x)$ is a point 
  mass for all $x$) and the generalized measure is such that
  $\operatorname{sign}(\ell(x,y) \cdot \nabla g_l(R))$
  is constant for all $(x,y) \in \Xspace \times \Yspace$ and 
  $l \in \{1,\ldots,m\}$.
  Then the asymptotically-optimal proposal achieves 
  $\trace \Sigma = 0$.
\end{proposition}

\begin{proof}
  From Proposition~\ref{prop:asymp-opt-proposal}, the asymptotically optimal 
  proposal achieves a total variance of 
  \begin{equation}
    \trace \Sigma = \left(\int v(x) \, \dd x\right)^2 - \|\Jac_g(R) \, R\|_2^2.
  \end{equation}
  We evaluate the two terms in this expression separately.
  Using the fact that $p(y|x) = \ind{y = y(x)}$, the first term 
  becomes
  \begin{align*}
    \left(\int v(x) \, dx \right)^2 & = 
        \left( \int \| \Jac_g(R) \, \ell(x,y(x)) \|_2 \, p(x) \, \dd x \right)^2 \\
      & \leq \left( \int \| \Jac_g(R) \, \ell(x,y(x)) \|_1 \, p(x) \, \dd x  \right)^2 \\
      & = \left( \sum_{l = 1}^{m} \int \ell(x, y(x)) \cdot \nabla g_l(R) \, p(x) \, \dd x  \right)^2.
  \end{align*}
  The second line follows by application of the inequality 
  $\|x\|_2 \leq \| x\|_1$, and the third line follows by assumption. 
  For the second term we have
  \begin{align*}
    \| \Jac_g(R) \, R \|_2^2 & = \left\| \int \Jac_g(R) \, \ell(x, y(x)) \, p(x) \, \dd x \right\|_2^2 \\
      & = \sum_{l = 1}^{m} \left( \int \ell(x, y(x)) \cdot \nabla g_l(R) \, p(x) \, \dd x \right)^2 \\
      & \geq \left( \sum_{l = 1}^{m} \int \ell(x, y(x)) \cdot \nabla g_l(R) \, p(x) \, \dd x \right)^2,
  \end{align*}
  by application of Jensen's inequality.
  Subtracting the second term from the first, we have 
  $\trace \Sigma \leq 0$.
\end{proof}

By way of illustration, we apply the above proposition to two common 
performance measures: accuracy and recall. 
We assume a deterministic oracle in both cases.

\begin{example}[Asymptotic variance for accuracy]
    From Table~\ref{tbl:classification-measures}, we have that 
    accuracy can be expressed as a generalized performance measure by 
    setting $\ell(x, y) = \ind{y \neq f(x)}$ and $g(R) = 1 - R$.
    Evaluating the condition in Proposition~\ref{app-prop:zero-tot-var}, 
    we have
    \begin{equation*}
        \operatorname{sign} \left( \ell(x,y) \cdot \nabla g(R) \right) 
            = \operatorname{sign} \left( - \ind{y \neq f(x)} \right) 
            = -1
    \end{equation*}
    for all $(x, y) \in \Xspace \times \Yspace$. 
    Thus our framework can achieve zero asymptotic total variance when 
    estimating accuracy under a deterministic oracle.
\end{example}

\begin{example}[Asymptotic variance for recall]
  From Table~\ref{tbl:classification-measures}, we have that 
  recall can be expressed as a generalized performance measure by 
  setting $\ell(x, y) = [y f(x), y]^\intercal$ and $g(R) = R_1 / R_2$.
  The conditions of Proposition~\ref{app-prop:zero-tot-var} are 
  not satisfied in this case, since 
  \begin{equation*}
      \operatorname{sign} \left( \ell(x, y) \cdot \nabla g(R) \right) 
          = \operatorname{sign} \left( \frac{y}{R_2} (f(x) - G_\mathrm{rec}) \right) 
          = \operatorname{sign} \left( f(x) - G_\mathrm{rec} \right) 
  \end{equation*}
  which is not constant for all $(x, y) \in \Xspace \times \Yspace$. 
  Indeed, when we evaluate the expression for the asymptotic total 
  variance (see Proposition~\ref{prop:asymp-opt-proposal}), we find that 
  $\Sigma = 4 G_\mathrm{rec}^2 (1 - G_\mathrm{rec})^2$.
  Therefore in general there is positive lower bound on the asymptotic 
  variance that can be achieved by our framework when estimating 
  recall under a deterministic oracle.
\end{example}

\section{Proof of Proposition~\ref{prop:valid-proposal-est}}

Before proving the proposition, we establish a useful corollary.

\begin{corollary}
\label{app-cor:corollary}
Suppose the generalized measure $G$ is defined with respect to a 
finite input space $\Xspace$ (e.g.~a finite pool of test data). 
\begin{enumerate}[label=(\roman*)]
  \item If the support of proposal $q_j(x, y) = q_j(x) p(y|x)$ is a superset 
    of $\{x, y \in \Xspace \times \Yspace : p(x, y) \|\ell(x, y)\| \neq 0\}$ 
    for all $j \geq 0$, then Theorem~\ref{thm:ais-consistency} holds. 
  \item If in addition $q_j(x) \overset{\mathrm{a.s.}}{\to} q_\infty(x)$ 
     pointwise in $x$, then Theorem~\ref{thm:ais-clt} holds.
\end{enumerate}
\end{corollary}

\begin{proof}
For the first statement, we check conditions \eqref{eqn:finite-var} 
and \eqref{eqn:asymp-negligible} of 
Theorem~\ref{thm:ais-consistency}.
Let $\mathcal{Q}_j \subset \Xspace$ be the support of $q_{j}(x)$ and let 
$\delta_j = \inf_{x \in \mathcal{Q}_j} q_{j}(x) > 0$.
For $\eta \geq 0$ we have
\begin{align*}
  \E\left[ \left(\frac{p(X_j)}{q_{j-1}(X_j)}\right)^{2 + \eta} \middle| \mathcal{F}_{j-1} \right] 
  & = \int \sum_{x \in \mathcal{Q}_{j- 1}} 
      \frac{p(x)^{2 + \eta} q_{j-1}(x) p(y|x)}{q_{j-1}(x)^{2 + \eta}}  \, \dd y \\
  & \leq \left(\frac{1}{\delta_j} \right)^{1 + \eta} < \infty.
\end{align*}

For the second statement, we must additionally check condition 
\eqref{eqn:V-convergence} regarding the convergence of $V_j$.
Letting
\begin{equation*}
    f_j(x, y) = \left( \frac{p(x)}{q_j(x)} \ell(x, y) - \risk \right)
    \left( \frac{p(x)}{q_j(x)} \ell(x, y) - \risk \right)^\intercal q_j(x) p(y|x),
\end{equation*}
we can write $V_j = \int \sum_{x \in \mathcal{Q}_j} f_j(x,y) \, \dd y$.
By the a.s.\ pointwise convergence of $q_j(x)$ and the continuous mapping 
theorem, we have $f_j(x,y) \to f_\infty(x,y)$ a.s.\ pointwise in $x$ and $y$. 
Now observe that
\begin{align*}
    \| f_j(x,y) \|_2 &= q_j(x, y) \left\| \frac{p(x)}{q_j(x)} \ell(x,y) - R \right\|_2^2 \\
    & \leq q_j(x, y) \left( \frac{p(x, y)^2}{q_j(x, y)^2} \| \ell(x,y) \|_2^2 + \|R \|_2^2 \right) \\
    & \leq p(y|x) \left( \frac{1}{\epsilon^2} \| \ell(x,y) \|_2^2 + \|R \|_2^2 \right) = h(x,y).
\end{align*}
It is straightforward to show that 
$\int \sum_{x \in \mathcal{Q}_j} h(x, y) \, \dd y < \infty$ using 
the boundedness of $\ell(x,y)$ (see Definition~\ref{def:gen-measure}). 
Thus we have $V_j \to V_\infty$ by the dominated convergence theorem.
\end{proof}

We can now prove Proposition~\ref{prop:valid-proposal-est} by showing that 
the conditions of Corollary~\ref{app-cor:corollary} hold. 
We focus on the case of a deterministic oracle---the proof for 
a stochastic oracle follows by a similar argument.

First we examine the support of the sequence of proposals. 
At stage $t$, the proposal can be expressed as
\begin{gather*}
  q_t(x) = \frac{v_t(x)}{\sum_{x \in \Xspace} v_t(x)} 
  \quad \text{with} \\
  v_t(x) = p(x) \int \max \left\{ 
      \| \Jac_g(\hat{R}_t) \, \ell(x, y)\|_2, 
      \epsilon_t \ind{\|\ell(x, y)\| \neq 0} \right\} \, \pi_t(y|x) \, \dd y.
\end{gather*}
Observe that 
\begin{equation*}
  v_t(x) \geq \epsilon_t \, p(x) \int \ind{\|\ell(x, y)\| \neq 0} \, \pi_t(y|x) \, \dd y   
\end{equation*}
and 
\begin{equation*}
  \begin{split}
    v_t(x) & \leq p(x) \int \{ \epsilon_t + \| \Jac_g(\hat{R}_t) \|_2 
        \, \|\ell(x, y)\|_2 \} \pi_t(y|x) \, \dd y \\
      & \leq p(x) \left( \epsilon_t + d^2 K \sup_{x, y \in \Xspace \times \Yspace} 
        \, \|\ell(x, y)\|_\infty \right) \\
      & \leq C p(x) 
  \end{split}
\end{equation*}
where $C < \infty$ is a constant. 
The upper bound follows from the boundedness of $\ell(x,y)$ (see 
Definition~\ref{def:gen-measure}), the boundedness of $\epsilon_t$, and 
the boundedness of the Jacobian.
Since 
\begin{equation*}
  \sum_{x \in \Xspace} v_t(x) \geq \epsilon_t \sum_{x \in \Xspace} 
    p(x) \int \ind{\|\ell(x, y)\| \neq 0} \pi_t(y|x) \, \dd y > 0
\end{equation*}
by assumption and $v_t(x)$ is bounded from above,
we conclude that $q_t(x)$ is a valid distribution for all $t \geq 0$. 
The lower bound on $v_t(x)$ implies that the support of 
$q_t(x, y) = q_t(x) p(y|x)$ is 
\begin{gather*}
  \{(x, y) \in \Xspace \times \Yspace : p(x, y) \pi_t(y|x) \|\ell(x, y)\| \neq 0 \} \\
  \supseteq \{(x, y) \in \Xspace \times \Yspace : p(x, y) \|\ell(x, y)\| \neq 0 \}.
\end{gather*}
The inequality follows from the fact that the support of $\pi_t(y|x)$ 
includes the support of $p(y|x) = \ind{y = y(x)}$.
Thus $q_t(x)$ has the required support for all $t \geq 0$.

Next we prove that the sequence of proposals converges a.s.\ pointwise in $x$.
Given that $\hat{R}_t \overset{\mathrm{a.s.}}{\to} \hat{R}_\infty$ 
and $\pi_t(y|x) \overset{\mathrm{a.s.}}{\to} \pi_\infty(y|x)$ pointwise 
in $x$, one can show by application of the continuous mapping theorem 
and dominated convergence theorem that 
\begin{equation*}
  v_t(x) \overset{\mathrm{a.s.}}{\to} v_\infty(x) = p(x) \int \| \Jac_g(R_\infty) \, \ell(x, y) \|_2 \pi_\infty(y|x) \, \dd y
\end{equation*}
pointwise in $x$.
By application of the continuous mapping theorem, we then have that 
$q_t(x) \overset{\mathrm{a.s.}}{\to} q_\infty(x) = \frac{v_\infty(x)}{\sum_{x \in \Xspace} v_\infty(x)}$.
Thus Theorems~\ref{thm:ais-consistency} and~\ref{thm:ais-clt} hold. 
Furthermore, if $\hat{R}_\infty = R$ and $\pi_\infty(y|x) = \ind{y = y(x)}$, 
then $q_\infty(x)$ is equal to the asymptotically-optimal proposal 
$q^\star(x)$ as defined in \eqref{eqn:asym-opt-proposal}. 
\qed

\section{Inference for the Dirichlet-tree model in Section~\ref{sec:model-proposal}}
\label{app-sec:oracle-est-det}

In this appendix, we outline a procedure for inferring the oracle response 
based on the hierarchical model presented in Section~\ref{sec:model-proposal}.
Recall that the model assumes a deterministic response---i.e.\ there is only 
one possible response (label) $y$ for a given input $x$. 
At stage $t$ of the evaluation process (see Algorithm~\ref{alg:ais}), the 
labels for instances in the test pool are partially-observed. 
To estimate the response for the unobserved instances, we can apply the 
expectation-maximization (EM) algorithm. 
Omitting the dependence on $t$, we let $\vec{y}_{(\vec{o})}$ denote 
the observed labels and $\vec{y}_{(\neg \vec{o})}$ denote the unobserved 
labels. 
The EM~algorithm returns a distribution over the unobserved labels 
$\vec{y}_{(\neg \vec{o})}$ and MAP estimates of the model parameters 
$\vec{\phi} = (\theta, \psi)$. 
At each iteration $\tau$ of the EM~algorithm, the following two steps 
are applied:
\begin{itemize}
  \item \textbf{E-step.} 
  Compute the function
  \begin{equation}
    Q(\vec{\phi} | \vec{\phi}^{(\tau)}) = \E_{\vec{y}_{(\neg \vec{o})} | \vec{y}_{\vec{o}}, \vec{k}, \vec{\phi}^{(\tau)}}\left(\log p(\vec{\phi} | \vec{y}, \vec{k}) \right),
    \label{app-eqn:em-Q-def}
  \end{equation}
  which is the expected log posterior with respect to the current 
  distribution over the unobserved labels $\vec{y}_{(\neg \vec{o})}$,  
  conditional on the observed labels $\vec{y}_{(\vec{o})}$ and the 
  current parameter estimates $\vec{\phi}^{(\tau)}$.
  \item \textbf{M-step.} 
  Update the parameter estimates by maximizing $Q$:
  \begin{equation}
    \vec{\phi}^{(\tau+1)} \in \arg \max_{\vec{\phi}} Q(\vec{\phi} | \vec{\phi}^{(\tau)}).
  \end{equation}
\end{itemize}

In order to implement the E- and M-steps, we must evaluate the $Q$ function 
for our model. 
Since the Dirichlet prior on $\theta$ and Dirichlet-tree priors on 
$\phi_y$ are conjugate to the categorical distribution, the posterior 
$p(\vec{\phi} | \vec{y}, \vec{k})$ is straightforward to compute.
We have 
\begin{align*}
  \theta | \vec{y}, \alpha &\sim \operatorname{Dirichlet}
    ( \tilde{\alpha} ), \\
  \psi_y | \vec{y}, \vec{k}, \beta_y, T &\sim \operatorname{DirichletTree} 
    ( \tilde{\beta}_y, T ), 
\end{align*}
where $\tilde{\alpha}_y = \alpha_y + \sum_{i = 1}^{M} \ind{y_i = y}$,  
$\tilde{\beta}_{y \nu}= \beta_{y \nu} + \sum_{i = 1}^{M} \ind{y_i = y} 
\delta_{\nu}(k_i)$ and $\delta_{\nu}(k)$ is defined in 
\eqref{app-eqn:delta-defn}. The posterior density for $\theta$ is
\begin{equation*}
  p(\theta | \vec{y}, \alpha) \propto \prod_{y = 1}^{C} \theta_{y}^{\tilde{\alpha}_y - 1}.
\end{equation*}
\citet{minka_dirichlet-tree_1999} gives the density for 
$\psi_y$ as:
\begin{align*}
  p(\psi_y|\vec{y}, \vec{k}, \vec{\beta}_y, T) \propto 
    \prod_{k \in \mathrm{lv}(T)} \psi_{y k}^{\tilde{\beta}_{y k} - 1}
      \prod_{\nu \in \mathrm{in}(T)} \left(\sum_{k \in \mathrm{lv}(T)} \sum_{\nu' \in \mathrm{children}(\nu)} \delta_{\nu'}(k) \psi_{y k}\right)^{\gamma_{y \nu}},
\end{align*}
where $\mathrm{lv}(T)$ denotes the set of leaf nodes in $T$, 
$\mathrm{in}(T)$ denotes the set of inner nodes in $T$, 
$\mathrm{lv}(\nu)$ denotes the leaf nodes reachable from node $\nu$, 
and $\tilde{\gamma}_{y \nu} = \tilde{\beta}_{y \nu} - 
\sum_{c \in \mathrm{children}(\nu)} \tilde{\beta}_{y c}$.

Substituting the posterior densities in \eqref{app-eqn:em-Q-def}, we have
\begin{align*}
  Q(\vec{\phi}|\vec{\phi}^{(t)})
  & = \sum_{y \in \Yspace} \sum_{k \in \mathrm{lv(T)}} 
    \left(\tilde{\beta}_{y k}^{(\tau)} - 1 \right) \log \psi_{k y} 
      + \sum_{y \in \Yspace} \sum_{\nu \in \mathrm{in}(T)} 
        \tilde{\gamma}_{y \nu}^{(\tau)} \log \left( \sum_{k \in \mathrm{lv}(T)} \sum_{\nu' \in \mathrm{children}(\nu)} \delta_{\nu'}(k) \psi_{y k} \right) \\
  & \qquad {} + \sum_{y \in \Yspace} \left( \tilde{\alpha}_y^{(\tau)} - 1 \right) \log \theta_{y} + \mathrm{const}.
\end{align*}
where we define 
$\tilde{\beta}_{y k}^{(\tau)} = \E_{\vec{y}_{(\neg \vec{o})} | \vec{y}_{\vec{o}}, \vec{k}, \vec{\phi}^{(\tau)}}[\tilde{\beta}_{y k}]$ 
and similarly for $\tilde{\alpha}_y^{(\tau)}$ and 
$\tilde{\gamma}_{y \nu}^{(\tau)}$.
When maximizing $Q(\vec{\phi}|\vec{\phi}^{(t)})$ with respect to $\vec{\phi}$, 
we must obey the constraints:
\begin{itemize}
  \item $\theta_y > 0$ for all $y \in \Yspace$, 
  \item $\sum_{y \in \Yspace} \theta_y = 1$, 
  \item $\psi_{y k} > 0$ for all $y \in \Yspace$ and leaf nodes 
  $k \in \mathrm{lv}(T)$, and
  \item $\sum_{k \in \mathrm{lv}(T)} \psi_{y k} = 1$.
\end{itemize}
We can maximize $\theta$ and $\{\psi_y\}$ separately since they are 
independent. 
For $\theta_y$ we have the mode of a Dirichlet random variable:
\begin{equation*}
  \theta_y^{(\tau + 1)} = \frac{\tilde{\alpha}_{y}^{(\tau)} - 1}
    {\sum_{y'} \{\tilde{\alpha}_{y'}^{(\tau)} - 1\}} 
\end{equation*}
and for $\psi_y$ we have (see \citealp{minka_dirichlet-tree_1999}):
\begin{gather}
  \psi_{y k}^{(\tau + 1)} = \prod_{\nu \in \mathrm{in}(T)} 
      \prod_{c \in \mathrm{children}(\nu)} \left( b_{y c}^{(\tau + 1)} \right)^{\delta_c(k)} \label{app-eqn:em-psi-update} \\
  \text{where} \quad b_{y c}^{(\tau + 1)} = \frac{\tilde{\beta}_{y c}^{(\tau)} - \sum_{k \in \mathrm{lv}(T)} \delta_{c}(k)}
  {\sum_{c' \in \mathrm{siblings}(c) \cup \{c\}} \{\tilde{\beta}_{y c'}^{(\tau)} - \sum_{k \in \mathrm{lv}(T)} \delta_{c'}(k)\}}.
  \label{app-eqn:em-branch-prob-update}
\end{gather}
The parameters $\{b_{y c}: c \in \mathrm{children}(\nu) \}$ may be 
interpreted as \emph{branching probabilities} for node 
$\nu \in \mathrm{in}(T)$.

In summary, the EM~algorithm reduces to the following two steps:
\begin{itemize}
  \item \textbf{E-step.} Compute the expected value for each unobserved 
  label using $\vec{\phi}^{(\tau)}$:
  \begin{equation}
    \E[\ind{y_j = y} | k_j, \vec{\phi}^{(\tau)}] = 
      \frac{\psi_{y k_j}^{(\tau)} \theta_{y}^{(\tau)}}
        {\sum_{y' \in \Yspace} \psi_{y' k_j}^{(\tau)} \theta_{y'}^{(\tau)}}.
    \label{app-eqn:em-exp-label}
  \end{equation}
  Then make a backward pass through the tree, computing 
  $\tilde{\beta}_{y \nu}^{(\tau)}$ at each internal node 
  $\nu \in \mathrm{in}(T)$.
  \item \textbf{M-step.} Make a forward-pass through the tree, updating the 
  branch probabilities $b_{y c}^{(\tau + 1)}$ using 
  \eqref{app-eqn:em-branch-prob-update}. 
  Compute $\psi_y^{(\tau + 1)}$ at the same time using 
  \eqref{app-eqn:em-psi-update}.
\end{itemize}
We can interpret \eqref{app-eqn:em-exp-label} as providing a posterior 
estimate for the unknown oracle response: 
\begin{equation}
  \pi(y|x) \propto \psi_{y k_x}^{(\tau)} \theta_y^{(\tau)}
  \label{app-eqn:pi-em}
\end{equation}
where $k_x$ 
denotes the assigned stratum for instance $x$. 
If the response $y(x)$ for instance $x$ has been observed in a previous 
stage of the evaluation process, then $\pi(y|x)$ collapses to a point 
mass at $y(x)$.

\section{Proof of Proposition~\ref{prop:valid-proposal-det}}
\label{app-sec:proof-valid-proposal-det}
Let $\pi_t(y|x)$ denote the posterior estimate of the (determinisitc) 
oracle response at stage $t$, as defined in \eqref{app-eqn:pi-em}.
First we must ensure that the support of $\pi_t(y|x)$ includes the true 
label $y(x)$ for all $t \geq 0$. 
This condition is satisfied since the priors on $\theta$ and $\psi$ ensure 
$\theta_y > 0$ and $\psi_{k y} > 0$ for all $k \in \{1,\ldots, K\}$ and 
$y \in \Yspace$ (see~\ref{app-eqn:em-exp-label}). 
Once the label for instance $x$ is observed, the posterior degenerates to 
a point mass at the true value $y(x)$. 
This also implies that the sequence of proposals have the necessary support 
to ensure Theorem~\ref{thm:ais-clt} holds.

Second, we verify that $\pi_t(y|x)$ converges a.s.\ pointwise 
in $x$ and $y$ to a conditional pmf $\pi_\infty(y|x)$ independent of $t$. 
This condition is also satisfied, since the posterior $\pi_t(y|x)$ 
degenerates to a point mass at $y(x)$ once the label for instance 
$x$ is observed, and all instances are observed in the limit $t \to \infty$.
Thus $\pi_t(y|x) \overset{\mathrm{a.s.}}{\to} \ind{y = y(x)}$ and 
$\epsilon_t \downarrow 0$, which implies that the sequence of proposals 
converges to the asymptotically-optimal proposal.
\qed

\section{Extension of the Dirichlet-tree model to stochastic oracles}
\label{app-sec:oracle-est-stoc}

Recall that the Dirichlet-tree model in Section~\ref{sec:model-proposal} 
is tailored for a deterministic oracle---where $p(y|x)$ is a point mass at 
a \emph{single label} $y(x)$.
In this appendix, we extend the model to handle stochastic oracles, where 
$p(y|x)$ has support on \emph{multiple labels} in general.
The model retains the same structure, however we no longer assume 
there is a discrete set of items whose labels are either observed or 
unobserved. 
Instead, we allow for a potentially infinite number of item-label pairs 
(indexed by $j$ below) to be generated. 
In reality, these pairs correspond to labelled items drawn uniformly at 
random from the test pool $\testset$.
Since estimating the oracle response for individual instances requires 
estimating a large number of continuous parameters 
($|\testset| \times (|\Yspace| - 1)$ parameters), 
we instead opt to estimate the response averaged over instances at leaf 
nodes of the tree ($K \times (|\Yspace| - 1)$ parameters for $K$ leaf 
nodes).
We refer to this as the leaf-level oracle response $p_\mathrm{leaf}(y|k)$ 
below.

\paragraph{Model specification}
For clarity, we reintroduce the model here despite some overlap with 
Section~\ref{sec:model-proposal}.
The oracle response is modeled globally using a pmf 
$\theta = [\theta_1, \ldots, \theta_C]$ over the label space $\Yspace$ with a 
Dirichlet prior:
\begin{equation*}
  \theta | \alpha \sim \operatorname{Dirichlet} \! \left(\alpha\right),
\end{equation*}
where $\alpha = [\alpha_1, \ldots, \alpha_C] \in \reals_+^C$ are concentration 
hyperparameters. 
The label $y_j$ for each instance $j$ is then assumed to be generated 
i.i.d.\ according to $\theta$:
\begin{align*}
  y_j | \theta & \overset{\mathrm{iid.}}{\sim} \operatorname{Categorical} \! \left(\theta\right), 
    & j \in 1, \ldots, J.
\end{align*}
We assume a hierarchical partition of the space $\Xspace$ is given, 
encoded as a tree $T$.
Given $T$ and label $y_j$, we assume instance $j$ is assigned to a leaf node 
$k_j \in \{1, \ldots, K\}$ of $T$ according to a distribution $\psi_y$ with 
a Dirichlet-tree prior:
\begin{align*}
  \psi_{y} | \beta_y, T & \overset{\mathrm{ind.}}{\sim} \operatorname{DirichletTree}\! \left(\beta_y; T\right), 
    & y \in \Yspace, \\
  k_j | y_j, \psi_{y_j} & \overset{\mathrm{ind.}}{\sim} \operatorname{Categorical} \! \left(\psi_{y_j}\right),
    & j \in 1, \ldots, J.
\end{align*}
where $\beta_y$ is a set of Dirichlet concentration parameters associated with 
the internal nodes of $T$.

\paragraph{Inference}
To estimate the leaf-level oracle response $p_\mathrm{leaf}(y|k)$, we 
use the posterior predictive distribution
\begin{equation*}
  p(y_j | k_j, \mathcal{L}) = 
  \int_\psi \int_\theta p(y_j | k_j, \psi, \theta) p(\psi, \theta | \mathcal{L}), 
\end{equation*}
which encodes our uncertainty about the oracle response $y_j$ for a query 
instance $x_j$ from stratum $k_j$ conditional on the previously 
observed samples $\mathcal{L}$. 
If the observed samples $\mathcal{L}$ were collected through unbiased 
sampling (as assumed in the model above), we would compute the posterior 
predictive distribution as follows:
\begin{equation}
  \begin{split}
    p(y_j | k_j, \mathcal{L}) & \propto \int_\theta p(y_j | \theta) p(\theta | \mathcal{L}) \int_\psi p(k_j | y_j, \psi) p(\psi | \mathcal{L})  \\
      & \propto \int_\theta \theta_{y_j} p(\theta|\mathcal{L}) \int_\psi \psi_{y_j,k_j} p(\psi | \mathcal{L}) \\
      & \propto \tilde{\alpha}_{y_j} \times \prod_{\nu \in \mathrm{in}(T)} 
        \prod_{c \in \mathrm{children}(\nu)} \left( 
            \frac{\tilde{\beta}_{y_j c}}{\sum_{c' \in \mathrm{children}(\nu)} \tilde{\beta}_{y_j c'}} 
          \right)^{\delta_{c}(k_j)}
  \end{split}
  \label{app-eqn:oracle-est-hier-post-pred}
\end{equation}
where 
\begin{equation}
  \begin{split}
    \tilde{\alpha}_y &= \alpha_y + \sum_{(x', y', w') \in \mathcal{L}} \ind{y' = y}, \\
    \tilde{\beta}_{y c} & = \beta_{y c} + \sum_{(x', y', w') \in \mathcal{L}} \ind{y' = y} \delta_{c}(k_{x'}),
  \end{split}
  \label{app-eqn:oracle-est-hier-post-params}
\end{equation}
$\mathrm{in}(T)$ denotes the inner nodes of $T$, $\mathrm{children}(\nu)$ 
denotes the children of node~$\nu$, $k_{x} \in \{1, \ldots, K\}$ denotes the 
leaf index of instance $x$ and 
\begin{equation}
  \delta_{\nu}(k) := \begin{cases}
    1, & \text{if node $\nu$ is traversed to reach leaf node $k$}, \\
    0, & \text{otherwise.}
    \label{app-eqn:delta-defn}
  \end{cases}
\end{equation}
In words, the posterior parameters are updated by adding a count of `1' 
for every observation with label $y$ that is reachable from node $\nu$ 
in the tree $T$.
However, since the observed samples $\mathcal{L}$ are biased in our 
application, we must apply a bias-correction to 
\eqref{app-eqn:oracle-est-hier-post-params}:
\begin{align*}
  \tilde{\alpha}_y &= 
    \alpha_y + \sum_{(x', y', w') \in \mathcal{L}} w' \ind{y' = y}, \\
  \tilde{\beta}_{y c} &= 
    \beta_{y c} + \sum_{(x', y', w') \in \mathcal{L}} w' \ind{y' = y} \delta_{c}(k_{x'}).
\end{align*}
This guarantees that 
$\frac{\tilde{\alpha}_{y}}{N} \overset{\mathrm{a.s.}}{\to} \E[\ind{Y = y}]$ 
and 
$\frac{\tilde{\beta}_{y c}}{N} \overset{\mathrm{a.s.}}{\to} \E[\ind{Y = y} \delta_c(k_X)]$.

\begin{proposition}
  \label{prop:valid-proposal-stoch}
  Consider an instantiation of our evaluation framework (see 
  Algorithm~\ref{alg:ais}) under a stochastic oracle where:
  \begin{itemize}
    \item the oracle response for instance $x$ is estimated online using the 
    posterior predictive \eqref{app-eqn:oracle-est-hier-post-pred} for the 
    leaf node $k_x$ in which $x$ resides;
    \item the proposals are adapted using the estimator defined in 
    Proposition~\ref{prop:valid-proposal-est} with 
    $\epsilon_t = \epsilon_0 / (t + 1)$ for some user-specified 
    $\epsilon_0 > 0$; and 
    \item $\hat{R}_t = \frac{1}{M} \sum_{i = 1}^{M} \sum_{y \in \Yspace} \hat{p}_t(y|x_i) \ell(x_i, y).$
  \end{itemize}
  Then Theorems~\ref{thm:ais-consistency} and~\ref{thm:ais-clt} hold.
\end{proposition}

\begin{proof}
Let $\hat{p}_t(y|x)$ denote the posterior predictive in evaluated at stage $t$. 
First we must ensure that the support of $\hat{p}_t(y|x)$ is a superset 
of the the support of $p(y|x)$. 
This condition is satisfied since the priors on $\theta$ and $\psi$ 
ensure that $\hat{p}_t(y|x) > 0$ for all $y \in \Yspace$ 
(see~\ref{app-eqn:oracle-est-hier-post-pred}).
This implies that the proposals $\{q_t(x)\}$ have the necessary support 
to ensure that Theorem~\ref{thm:ais-consistency} is satisfied.

Second, we must ensure that $\hat{p}_t(y|x)$ converges a.s.\ pointwise 
in $x$ and $y$ to a conditional pmf $\hat{p}_\infty(y|x)$ independent of $t$. 
This condition is also satisfied, since the posterior parameters 
$\tilde{\alpha}_y / N$ and $\tilde{\beta}_{yc} / N$ converge a.s.\ to 
constants by Theorem~\ref{thm:ais-consistency} (see the text preceding 
the statement of this proposition).
\end{proof}

\section{Unsupervised partitioning methods}
\label{app-sec:unsupervised-partition}
The models for the oracle response described in 
Section~\ref{sec:model-proposal} require that the test pool be hierarchically 
partitioned into blocks. 
Ideally, the partition should be selected so that the oracle 
distribution $p(y|x)$ is approximately constant within each block. 
Since we begin without any labels, we are restricted to 
unsupervised partitioning methods. 
We briefly describe two settings for learning partitions: (i)~where 
classifier scores are used as a proxy for $p(y|x)$ and 
(ii)~where feature vectors are used. 

\paragraph{Score-based methods.} 
Our partitioning problem can be tackled using stratification
methods studied in the survey statistics community 
\citep{cochran_sampling_1977}. 
The aim of stratification is to partition a population into roughly 
homogenous subpopulations, known as strata (or blocks) for the purpose 
of variance reduction. 
When an auxiliary variable is observed that is correlated with the 
statistic of interest, the strata may be defined as a partition of 
the range of the auxiliary variable into sub-intervals. 
Various methods are used for determining the sub-intervals, 
including the cumulative square-root frequency (CSF) method 
\citep{dalenius_minimum_1959} and the geometric method 
\citep{gunning_new_2004}. 
In our application, the classifier scores are the auxiliary 
variables and the statistic of interest is $p(y = 1|x)$ 
(for binary classification).

Stratification is conventionally used to produce a partition without 
hierarchical structure. 
However, if the strata are ordered, it is straightforward to 
``fill in'' hierarchical structure.
In our experiments, we specify the desired tree structure---e.g.\
an ordered binary tree of a particular depth. 
We then compute the stratum bins (sub-intervals) so that the number of 
bins matches the number of leaf nodes in the tree. 
The stratum bins are then associated with the leaf nodes of the tree 
in breadth-first order.

\paragraph{Feature-based methods.}
When scores are not available, it is natural to consider unsupervised 
clustering methods which operate on the feature vectors. 
We expect $p(y|x)$ to be roughly constant within a cluster, since 
neighboring points in the feature space typically behave similarly. 
\citet{reddy_survey_2010} reviews hierarchical clustering methods 
including agglomerative and divisive clustering. 
One disadvantage of clustering methods, is that they tend to scale 
quadratically with the number of data points. 
A more efficient alternative is the $k$-d tree 
\citep{bentley_multidimensional_1975}.

\section{Additional convergence plots}
\label{app-sec:extra-plots}

\begin{figure}
  \centering
  \includegraphics[scale=1.3]{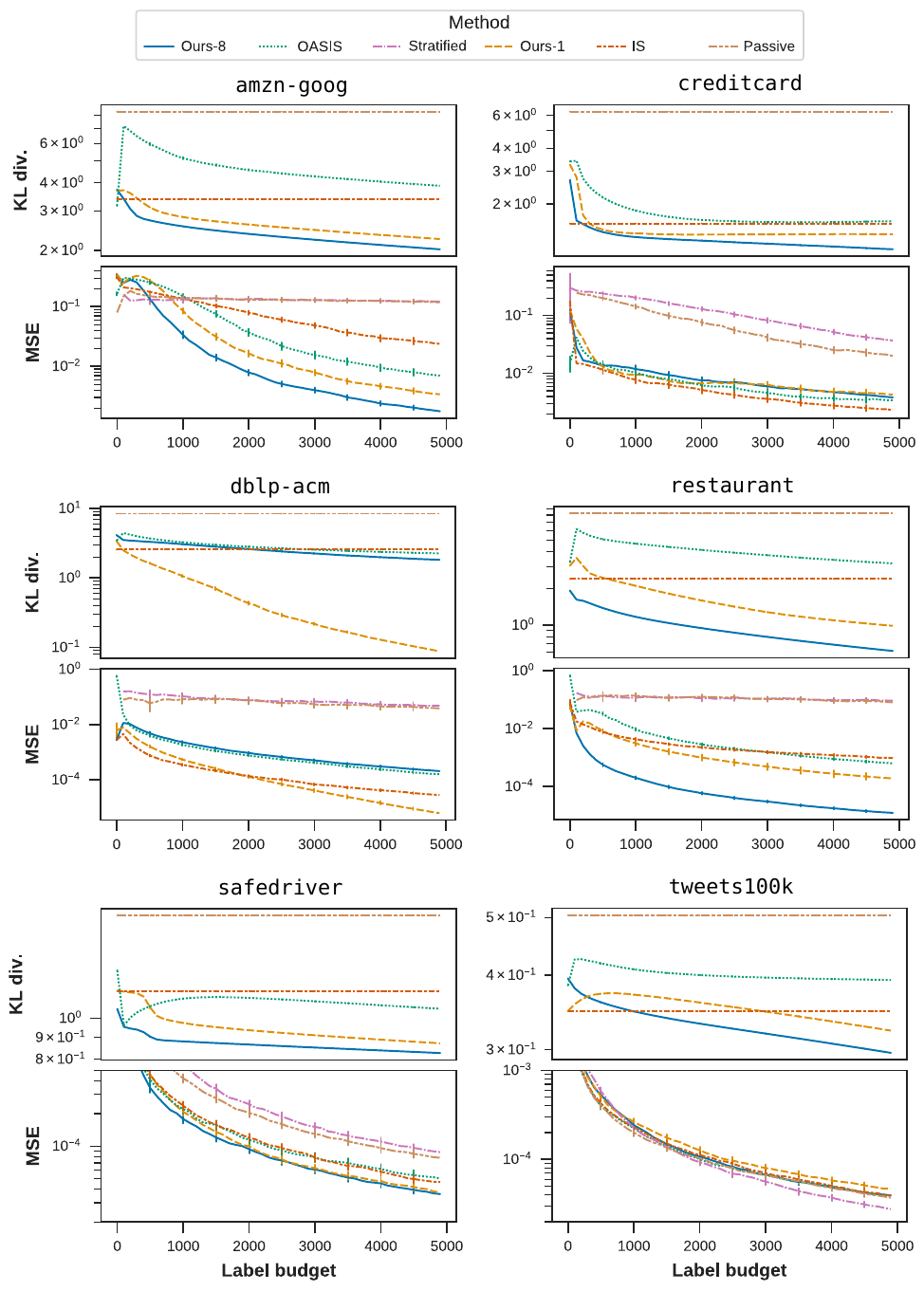}
  \caption{MSE of estimated F1 score over 1000~repeats as a function of 
  consumed label budget. 
  95\% bootstrap confidence intervals are included.}
  \label{app-fig:extra-convergence-plots}
\end{figure}

\paragraph{F1 score.} 
We provide additional convergence plots for estimating F1 score in 
Figure~\ref{app-fig:extra-convergence-plots}, which cover six of the seven 
datasets listed in Table~\ref{tbl:data-sets}.
The convergence plot for \texttt{abt-buy} is featured in the main paper 
in Figure~\ref{fig:convergence-abt-buy}.
The biased sampling methods (\texttt{Ours-8}, \texttt{Ours-1}, 
\texttt{OASIS} and \texttt{IS}) converge significantly more rapidly than 
\texttt{Passive} and \texttt{Stratified} for five of the six datasets.
\texttt{tweets100k} is an exception because it is the only dataset with 
well-balanced classes. 
Of the biased sampling methods, \texttt{Ours-8} performs best on 
two of the six datasets (\texttt{amzn-goog} and \texttt{restaurant}) and 
equal-best on one (\texttt{safedriver}). 
\texttt{Ours-1} performs best on one dataset (\texttt{dblp-acm}) and 
equal-best on one (\texttt{safedriver}), while
\texttt{IS} performs best on one dataset (\texttt{creditcard}).
In general, we expect \texttt{IS} to perform well when the oracle response 
$p(y|x)$ is already well-approximated by the model under evaluation. 
When this is not the case, the adaptive methods are expected to perform 
best as they produce refined estimates of $p(y|x)$ during sampling.

\paragraph{Accuracy.}
We repeated the experiments described in Section~\ref{sec:experiments} 
for estimating accuracy. 
Although accuracy is not recommended for evaluating classifiers in the 
presence of class imbalance, it is interesting to see whether the 
biased sampling methods offer any improvement in label efficiency, given the 
discussion in Section~\ref{sec:limitations-mc}.
Figures~\ref{app-fig:convergence-plots-accuracy-1}(a) 
and~\ref{app-fig:convergence-plots-accuracy-2} present 
convergence plots for the seven datasets listed in Table~\ref{tbl:data-sets}.
A summary of the MSE for all datasets is included in 
Figure~\ref{app-fig:convergence-plots-accuracy-1}(b) assuming a label 
budget of 1000. 
OASIS does not appear in these results, as it does not support estimation of 
accuracy.

We find that gains in label efficiency are less pronounced for the 
biased sampling methods when estimating accuracy. 
This is to be expected as accuracy is less sensitive to class imbalance, 
as noted in Section~\ref{sec:limitations-mc}. 
However, there is still a marked improvement in the convergence rate---by 
an order of magnitude or more---for three of the datasets 
(\texttt{abt-buy}, \texttt{dblp-acm} and \texttt{restaurant}). 
Again, we find that the more imbalanced datasets (\texttt{abt-buy}, 
\texttt{amzn-goog}, \texttt{dblp-acm} and  \texttt{restaurant}) 
seem to benefit most from biased sampling.

\paragraph{Precision-recall curves.}
Figure~\ref{app-fig:convergence-plots-pr-curve} presents convergence plots 
for estimating precision-recall curves for two of the test pools in 
Table~\ref{tbl:data-sets} assuming a label budget of~5000. 
We find that the biased sampling methods (\texttt{Ours-8}, 
\texttt{Ours-1} and \texttt{IS}) offer a significant improvement 
in the MSE compared to \texttt{Passive} and \texttt{Stratified}---by 
1--2~orders of magnitude.
The difference in the MSE between the AIS-based methods and \texttt{IS} 
is less pronounced here than when estimating F1-score. 

\begin{figure}
  \centering
  \raisebox{6cm}{(a)}\includegraphics[width=0.45\linewidth]{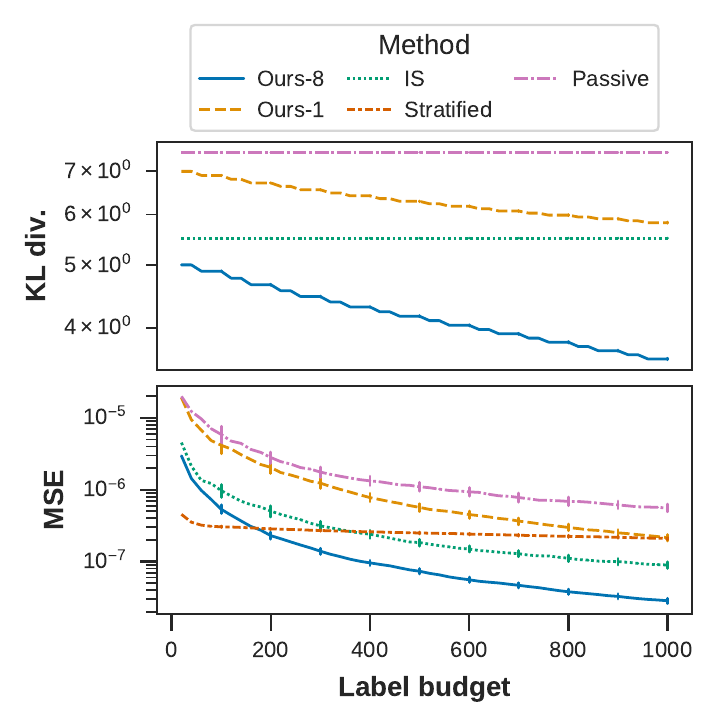}
  \hfill
  \raisebox{6cm}{(b)}\includegraphics[width=0.45\linewidth]{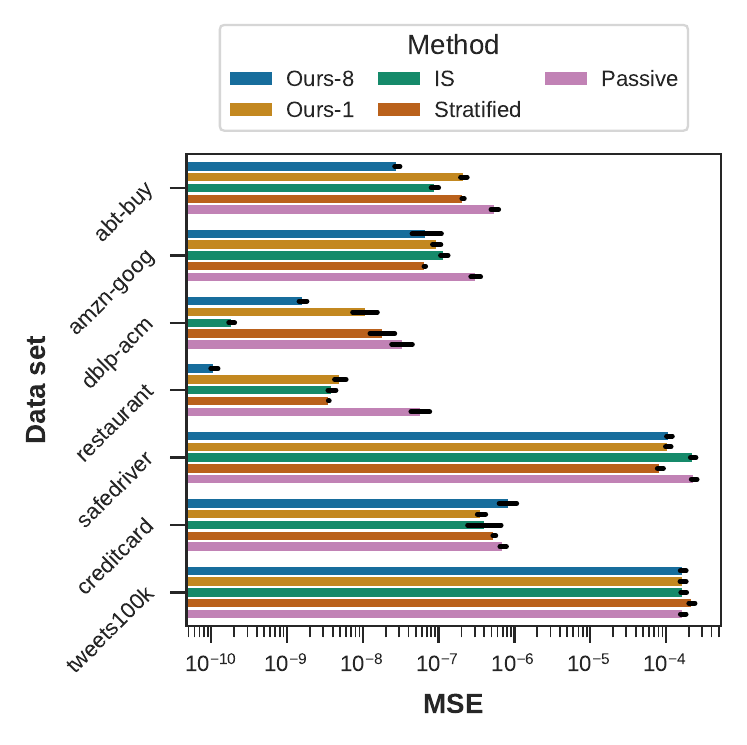}  
  \caption{(a)~Convergence plot for estimating accuracy on \texttt{abt-buy} 
  over 1000 repeats. 
  The upper panel plots the KL divergence from the proposal 
  to the asymptotically-optimal one. 
  The lower panel plots the MSE of the estimate for accuracy. 
  95\% bootstrap confidence intervals are included. 
  (b)~MSE of the estimate for accuracy after 1000 label queries. 
  95\% bootstrap confidence intervals are shown in black.}
  \label{app-fig:convergence-plots-accuracy-1}
\end{figure}

\begin{figure}
  \centering
  \includegraphics[scale=1.3]{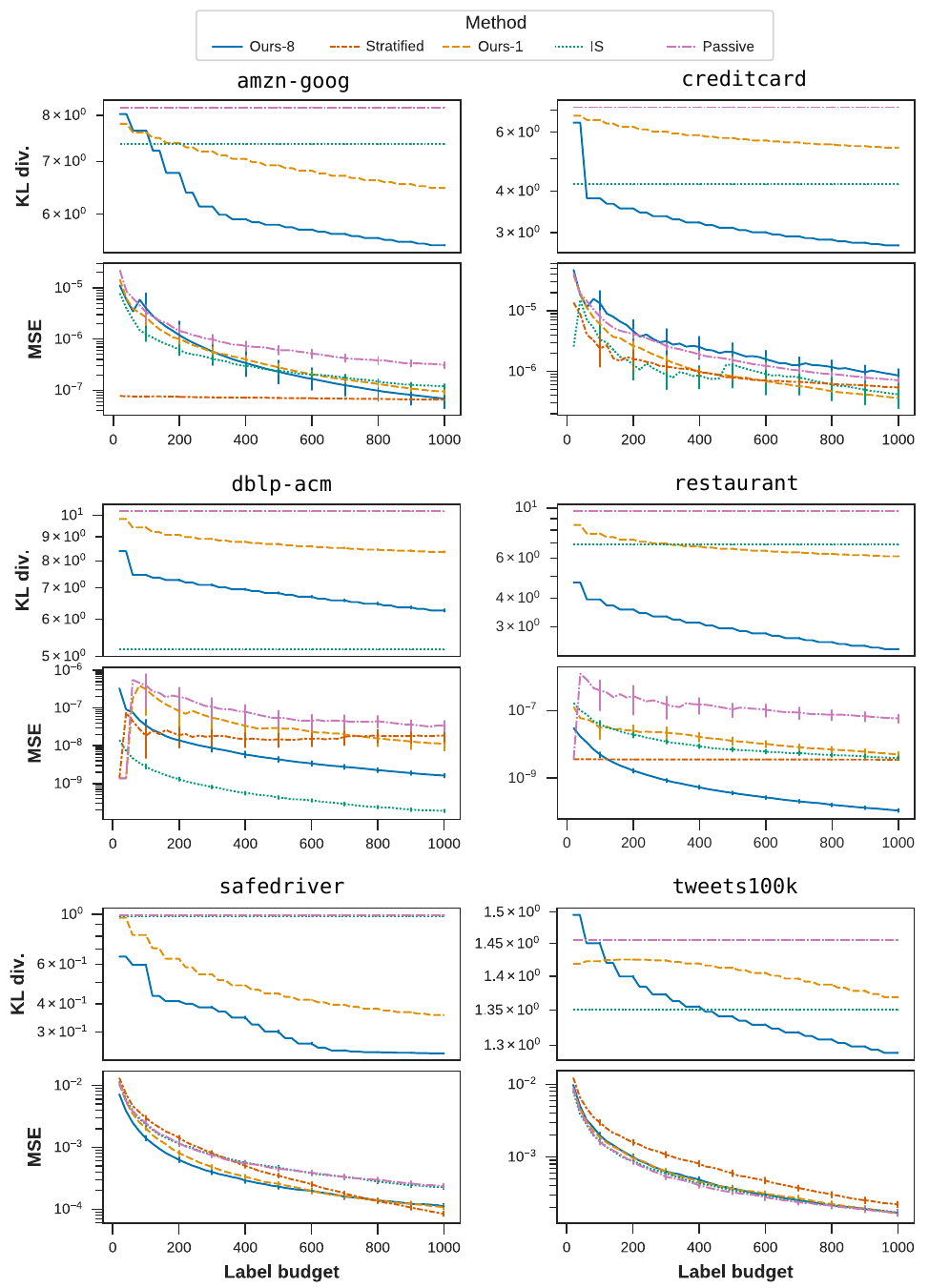}
  \caption{Convergence plots for estimating accuracy over 1000~repeats. 
  The upper panel in each sub-figure plots the KL~divergence from the 
  proposal to the asymptotically optimal one. 
  The lower panel plots the MSE of the estimate for accuracy.
  95\% bootstrap confidence intervals are included.}
  \label{app-fig:convergence-plots-accuracy-2}
\end{figure}

\begin{figure}
  \centering
  \includegraphics[scale=1.3]{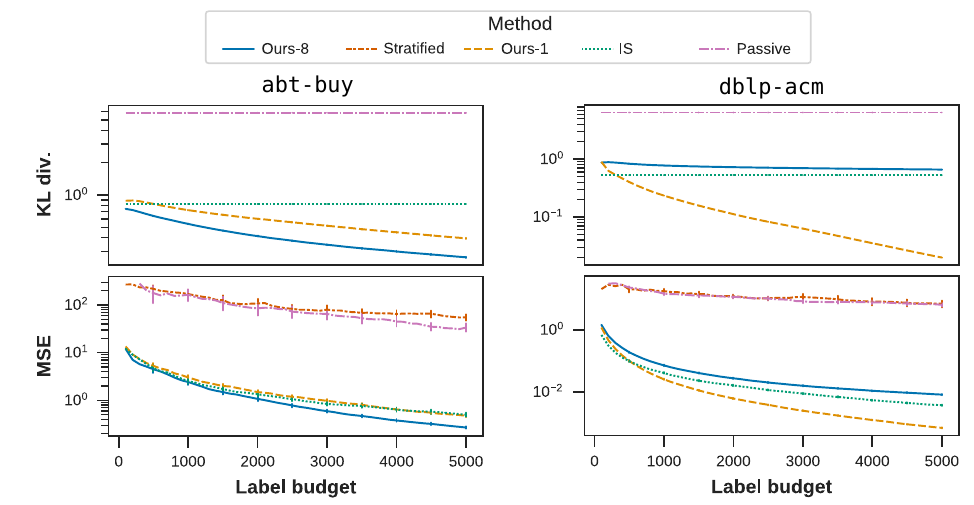}
  \caption{Convergence plots for estimating a precision-recall curve for 
    \texttt{abt-buy} and \texttt{dblp-acm} over 1000~repeats. 
    The upper panel in each sub-figure plots the KL~divergence from the 
    proposal to the asymptotically optimal one. 
    The lower panel plots the total MSE of the precision and recall 
    estimates at each threshold.
    95\% bootstrap confidence intervals are included.}
  \label{app-fig:convergence-plots-pr-curve}
\end{figure}

%% file: main.bbl

\begin{thebibliography}{41}


\ifx \showCODEN    \undefined \def \showCODEN     #1{\unskip}     \fi
\ifx \showDOI      \undefined \def \showDOI       #1{#1}\fi
\ifx \showISBNx    \undefined \def \showISBNx     #1{\unskip}     \fi
\ifx \showISBNxiii \undefined \def \showISBNxiii  #1{\unskip}     \fi
\ifx \showISSN     \undefined \def \showISSN      #1{\unskip}     \fi
\ifx \showLCCN     \undefined \def \showLCCN      #1{\unskip}     \fi
\ifx \shownote     \undefined \def \shownote      #1{#1}          \fi
\ifx \showarticletitle \undefined \def \showarticletitle #1{#1}   \fi
\ifx \showURL      \undefined \def \showURL       {\relax}        \fi
\providecommand\bibfield[2]{#2}
\providecommand\bibinfo[2]{#2}
\providecommand\natexlab[1]{#1}
\providecommand\showeprint[2][]{arXiv:#2}

\bibitem[\protect\citeauthoryear{Bennett and Carvalho}{Bennett and
  Carvalho}{2010}]%
        {bennett_online_2010}
\bibfield{author}{\bibinfo{person}{Paul~N. Bennett} {and}
  \bibinfo{person}{Vitor~R. Carvalho}.} \bibinfo{year}{2010}\natexlab{}.
\newblock \showarticletitle{{Online} {Stratified} {Sampling}: {Evaluating}
  {Classifiers} at {Web-scale}}. In \bibinfo{booktitle}{\emph{CIKM}}.
  \bibinfo{pages}{1581--1584}.
\newblock
\showISBNx{978-1-4503-0099-5}
\urldef\tempurl%
\url{https://doi.org/10.1145/1871437.1871677}
\showDOI{\tempurl}


\bibitem[\protect\citeauthoryear{Bentley}{Bentley}{1975}]%
        {bentley_multidimensional_1975}
\bibfield{author}{\bibinfo{person}{Jon~Louis Bentley}.}
  \bibinfo{year}{1975}\natexlab{}.
\newblock \showarticletitle{Multidimensional {Binary} {Search} {Trees} {Used}
  for {Associative} {Searching}}.
\newblock \bibinfo{journal}{\emph{Commun. ACM}} \bibinfo{volume}{18},
  \bibinfo{number}{9} (\bibinfo{date}{Sept.} \bibinfo{year}{1975}),
  \bibinfo{pages}{509--517}.
\newblock
\showISSN{0001-0782}
\urldef\tempurl%
\url{https://doi.org/10.1145/361002.361007}
\showDOI{\tempurl}


\bibitem[\protect\citeauthoryear{{Bugallo}, {Elvira}, {Martino}, {Luengo},
  {Miguez}, and {Djuric}}{{Bugallo} et~al\mbox{.}}{2017}]%
        {bugallo_adaptive_2017}
\bibfield{author}{\bibinfo{person}{M.~F. {Bugallo}}, \bibinfo{person}{V.
  {Elvira}}, \bibinfo{person}{L. {Martino}}, \bibinfo{person}{D. {Luengo}},
  \bibinfo{person}{J. {Miguez}}, {and} \bibinfo{person}{P.~M. {Djuric}}.}
  \bibinfo{year}{2017}\natexlab{}.
\newblock \showarticletitle{Adaptive {Importance} {Sampling}: {The} past, the
  present, and the future}.
\newblock \bibinfo{journal}{\emph{IEEE Signal Processing Magazine}}
  \bibinfo{volume}{34}, \bibinfo{number}{4} (\bibinfo{date}{July}
  \bibinfo{year}{2017}), \bibinfo{pages}{60--79}.
\newblock
\showISSN{1558-0792}
\urldef\tempurl%
\url{https://doi.org/10.1109/MSP.2017.2699226}
\showDOI{\tempurl}


\bibitem[\protect\citeauthoryear{Capp{\'e}, Douc, Guillin, Marin, and
  Robert}{Capp{\'e} et~al\mbox{.}}{2008}]%
        {cappe_adaptive_2008}
\bibfield{author}{\bibinfo{person}{Olivier Capp{\'e}}, \bibinfo{person}{Randal
  Douc}, \bibinfo{person}{Arnaud Guillin}, \bibinfo{person}{Jean-Michel Marin},
  {and} \bibinfo{person}{Christian~P. Robert}.}
  \bibinfo{year}{2008}\natexlab{}.
\newblock \showarticletitle{Adaptive importance sampling in general mixture
  classes}.
\newblock \bibinfo{journal}{\emph{Statistics and Computing}}
  \bibinfo{volume}{18}, \bibinfo{number}{4} (\bibinfo{date}{01 Dec.}
  \bibinfo{year}{2008}), \bibinfo{pages}{447--459}.
\newblock
\showISSN{1573-1375}
\urldef\tempurl%
\url{https://doi.org/10.1007/s11222-008-9059-x}
\showDOI{\tempurl}


\bibitem[\protect\citeauthoryear{Chen and Guestrin}{Chen and Guestrin}{2016}]%
        {chen_xgboost_2016}
\bibfield{author}{\bibinfo{person}{Tianqi Chen} {and} \bibinfo{person}{Carlos
  Guestrin}.} \bibinfo{year}{2016}\natexlab{}.
\newblock \showarticletitle{{XGBoost}: {A} {Scalable} {Tree} {Boosting}
  {System}}. In \bibinfo{booktitle}{\emph{Proceedings of the 22nd ACM SIGKDD
  International Conference on Knowledge Discovery and Data Mining}} (San
  Francisco, California, USA) \emph{(\bibinfo{series}{KDD '16})}.
  \bibinfo{publisher}{ACM}, \bibinfo{address}{New York, NY, USA},
  \bibinfo{pages}{785--794}.
\newblock
\showISBNx{978-1-4503-4232-2}
\urldef\tempurl%
\url{https://doi.org/10.1145/2939672.2939785}
\showDOI{\tempurl}


\bibitem[\protect\citeauthoryear{Cochran}{Cochran}{1977}]%
        {cochran_sampling_1977}
\bibfield{author}{\bibinfo{person}{William~G. Cochran}.}
  \bibinfo{year}{1977}\natexlab{}.
\newblock \bibinfo{booktitle}{\emph{Sampling {Techniques}}
  (\bibinfo{edition}{3rd} ed.)}.
\newblock \bibinfo{publisher}{Wiley}, \bibinfo{address}{New York}.
\newblock
\showISBNx{978-0-471-16240-7}


\bibitem[\protect\citeauthoryear{Cormack, Palmer, and Clarke}{Cormack
  et~al\mbox{.}}{1998}]%
        {cormack_efficient_1998}
\bibfield{author}{\bibinfo{person}{Gordon~V. Cormack},
  \bibinfo{person}{Christopher~R. Palmer}, {and} \bibinfo{person}{Charles L.~A.
  Clarke}.} \bibinfo{year}{1998}\natexlab{}.
\newblock \showarticletitle{Efficient {Construction} of {Large} {Test}
  {Collections}}. In \bibinfo{booktitle}{\emph{Proceedings of the 21st Annual
  International ACM SIGIR Conference on Research and Development in Information
  Retrieval}} (Melbourne, Australia) \emph{(\bibinfo{series}{SIGIR ’98})}.
  \bibinfo{publisher}{Association for Computing Machinery},
  \bibinfo{address}{New York, NY, USA}, \bibinfo{pages}{282--289}.
\newblock
\showISBNx{1581130155}
\urldef\tempurl%
\url{https://doi.org/10.1145/290941.291009}
\showDOI{\tempurl}


\bibitem[\protect\citeauthoryear{Cornuet, Marin, Mira, and Robert}{Cornuet
  et~al\mbox{.}}{2012}]%
        {cornuet_adaptive_2012}
\bibfield{author}{\bibinfo{person}{Jean-Marie Cornuet},
  \bibinfo{person}{Jean-Michel Marin}, \bibinfo{person}{Antonietta Mira}, {and}
  \bibinfo{person}{Christian~P. Robert}.} \bibinfo{year}{2012}\natexlab{}.
\newblock \showarticletitle{Adaptive {Multiple} {Importance} {Sampling}}.
\newblock \bibinfo{journal}{\emph{Scandinavian Journal of Statistics}}
  \bibinfo{volume}{39}, \bibinfo{number}{4} (\bibinfo{year}{2012}),
  \bibinfo{pages}{798--812}.
\newblock
\urldef\tempurl%
\url{https://doi.org/10.1111/j.1467-9469.2011.00756.x}
\showDOI{\tempurl}


\bibitem[\protect\citeauthoryear{Dalenius and Hodges}{Dalenius and
  Hodges}{1959}]%
        {dalenius_minimum_1959}
\bibfield{author}{\bibinfo{person}{Tore Dalenius} {and}
  \bibinfo{person}{Joseph~L. Hodges}.} \bibinfo{year}{1959}\natexlab{}.
\newblock \showarticletitle{Minimum {Variance} {Stratification}}.
\newblock \bibinfo{journal}{\emph{J. Amer. Statist. Assoc.}}
  \bibinfo{volume}{54}, \bibinfo{number}{285} (\bibinfo{date}{March}
  \bibinfo{year}{1959}), \bibinfo{pages}{88--101}.
\newblock
\showISSN{0162-1459}
\urldef\tempurl%
\url{https://doi.org/10.1080/01621459.1959.10501501}
\showDOI{\tempurl}


\bibitem[\protect\citeauthoryear{Dennis}{Dennis}{1996}]%
        {dennis_bayesian_1996}
\bibfield{author}{\bibinfo{person}{Samuel~Y. Dennis}.}
  \bibinfo{year}{1996}\natexlab{}.
\newblock \showarticletitle{A {Bayesian} analysis of tree-structured
  statistical decision problems}.
\newblock \bibinfo{journal}{\emph{Journal of Statistical Planning and
  Inference}} \bibinfo{volume}{53}, \bibinfo{number}{3} (\bibinfo{year}{1996}),
  \bibinfo{pages}{323--344}.
\newblock
\showISSN{0378-3758}
\urldef\tempurl%
\url{https://doi.org/10.1016/0378-3758(95)00112-3}
\showDOI{\tempurl}


\bibitem[\protect\citeauthoryear{Douc and Moulines}{Douc and Moulines}{2008}]%
        {douc_limit_2008}
\bibfield{author}{\bibinfo{person}{Randal Douc} {and} \bibinfo{person}{Eric
  Moulines}.} \bibinfo{year}{2008}\natexlab{}.
\newblock \showarticletitle{Limit theorems for weighted samples with
  applications to sequential {Monte Carlo} methods}.
\newblock \bibinfo{journal}{\emph{The Annals of Statistics}}
  \bibinfo{volume}{36}, \bibinfo{number}{5} (\bibinfo{date}{Oct.}
  \bibinfo{year}{2008}), \bibinfo{pages}{2344--2376}.
\newblock
\urldef\tempurl%
\url{https://doi.org/10.1214/07-AOS514}
\showDOI{\tempurl}


\bibitem[\protect\citeauthoryear{Druck and {McCallum}}{Druck and
  {McCallum}}{2011}]%
        {druck_toward_2011}
\bibfield{author}{\bibinfo{person}{Gregory Druck} {and} \bibinfo{person}{Andrew
  {McCallum}}.} \bibinfo{year}{2011}\natexlab{}.
\newblock \showarticletitle{{Toward} {Interactive} {Training} and
  {Evaluation}}. In \bibinfo{booktitle}{\emph{CIKM}} (New York, {NY}, {USA}).
  \bibinfo{pages}{947--956}.
\newblock
\showISBNx{978-1-4503-0717-8}
\urldef\tempurl%
\url{https://doi.org/10.1145/2063576.2063712}
\showDOI{\tempurl}


\bibitem[\protect\citeauthoryear{Feller}{Feller}{1968}]%
        {feller_introduction_1968}
\bibfield{author}{\bibinfo{person}{W. Feller}.}
  \bibinfo{year}{1968}\natexlab{}.
\newblock \bibinfo{booktitle}{\emph{An {Introduction} to {Probability} {Theory}
  and {Its} {Applications}, {Volume 1}} (\bibinfo{edition}{3rd} ed.)}.
\newblock \bibinfo{publisher}{Wiley}.
\newblock
\showISBNx{978-0-471-25708-0}


\bibitem[\protect\citeauthoryear{Feller}{Feller}{1971}]%
        {feller_introduction_1971}
\bibfield{author}{\bibinfo{person}{W. Feller}.}
  \bibinfo{year}{1971}\natexlab{}.
\newblock \bibinfo{booktitle}{\emph{An {Introduction} to {Probability} {Theory}
  and {Its} {Applications}, {Volume 2}} (\bibinfo{edition}{2nd} ed.)}.
\newblock \bibinfo{publisher}{Wiley}.
\newblock
\showISBNx{978-0-471-25709-7}


\bibitem[\protect\citeauthoryear{Gao, Li, Xu, Sisman, Dong, and Yang}{Gao
  et~al\mbox{.}}{2019}]%
        {gao_efficient_2019}
\bibfield{author}{\bibinfo{person}{Junyang Gao}, \bibinfo{person}{Xian Li},
  \bibinfo{person}{Yifan~Ethan Xu}, \bibinfo{person}{Bunyamin Sisman},
  \bibinfo{person}{Xin~Luna Dong}, {and} \bibinfo{person}{Jun Yang}.}
  \bibinfo{year}{2019}\natexlab{}.
\newblock \showarticletitle{Efficient Knowledge Graph Accuracy Evaluation}.
\newblock \bibinfo{journal}{\emph{Proc. VLDB Endow.}} \bibinfo{volume}{12},
  \bibinfo{number}{11} (\bibinfo{date}{July} \bibinfo{year}{2019}),
  \bibinfo{pages}{1679--1691}.
\newblock
\showISSN{2150-8097}
\urldef\tempurl%
\url{https://doi.org/10.14778/3342263.3342642}
\showDOI{\tempurl}


\bibitem[\protect\citeauthoryear{Gilyazev and Turdakov}{Gilyazev and
  Turdakov}{2018}]%
        {gilyazev_active_2018}
\bibfield{author}{\bibinfo{person}{R.~A. Gilyazev} {and}
  \bibinfo{person}{D.~Yu. Turdakov}.} \bibinfo{year}{2018}\natexlab{}.
\newblock \showarticletitle{{Active Learning and Crowdsourcing: A Survey of
  Optimization Methods for Data Labeling}}.
\newblock \bibinfo{journal}{\emph{Programming and Computer Software}}
  \bibinfo{volume}{44}, \bibinfo{number}{6} (\bibinfo{date}{Nov.}
  \bibinfo{year}{2018}), \bibinfo{pages}{476--491}.
\newblock
\showISSN{1608-3261}
\urldef\tempurl%
\url{https://doi.org/10.1134/S0361768818060142}
\showDOI{\tempurl}


\bibitem[\protect\citeauthoryear{Gunning and Horgan}{Gunning and
  Horgan}{2004}]%
        {gunning_new_2004}
\bibfield{author}{\bibinfo{person}{Patricia Gunning} {and}
  \bibinfo{person}{Jane~M. Horgan}.} \bibinfo{year}{2004}\natexlab{}.
\newblock \showarticletitle{A {New} {Algorithm} for the {Construction} of
  {Stratum} {Boundaries} in {Skewed} {Populations}}.
\newblock \bibinfo{journal}{\emph{Survey Methodology}} \bibinfo{volume}{30},
  \bibinfo{number}{2} (\bibinfo{date}{Dec.} \bibinfo{year}{2004}),
  \bibinfo{pages}{159--166}.
\newblock


\bibitem[\protect\citeauthoryear{Jaeger}{Jaeger}{2005}]%
        {jaeger_ignorability_2005}
\bibfield{author}{\bibinfo{person}{Manfred Jaeger}.}
  \bibinfo{year}{2005}\natexlab{}.
\newblock \showarticletitle{Ignorability in Statistical and Probabilistic
  Inference}.
\newblock \bibinfo{journal}{\emph{J. Artif. Int. Res.}} \bibinfo{volume}{24},
  \bibinfo{number}{1} (\bibinfo{date}{Dec.} \bibinfo{year}{2005}),
  \bibinfo{pages}{889–917}.
\newblock
\showISSN{1076-9757}


\bibitem[\protect\citeauthoryear{Kanoulas}{Kanoulas}{2016}]%
        {kanoulas_short_2016}
\bibfield{author}{\bibinfo{person}{Evangelos Kanoulas}.}
  \bibinfo{year}{2016}\natexlab{}.
\newblock \bibinfo{booktitle}{\emph{A {Short} {Survey} on {Online} and
  {Offline} {Methods} for {Search} {Quality} {Evaluation}}}.
\newblock \bibinfo{publisher}{Springer International Publishing},
  \bibinfo{address}{Cham}, \bibinfo{pages}{38--87}.
\newblock
\showISBNx{978-3-319-41718-9}
\urldef\tempurl%
\url{https://doi.org/10.1007/978-3-319-41718-9_3}
\showDOI{\tempurl}


\bibitem[\protect\citeauthoryear{Khalilia, Chakraborty, and Popescu}{Khalilia
  et~al\mbox{.}}{2011}]%
        {khalilia_predicting_2011}
\bibfield{author}{\bibinfo{person}{Mohammed Khalilia}, \bibinfo{person}{Sounak
  Chakraborty}, {and} \bibinfo{person}{Mihail Popescu}.}
  \bibinfo{year}{2011}\natexlab{}.
\newblock \showarticletitle{Predicting disease risks from highly imbalanced
  data using random forest}.
\newblock \bibinfo{journal}{\emph{BMC Medical Informatics and Decision Making}}
  \bibinfo{volume}{11}, \bibinfo{number}{1} (\bibinfo{year}{2011}),
  \bibinfo{numpages}{13}~pages.
\newblock
\showISSN{1472-6947}
\urldef\tempurl%
\url{https://doi.org/10.1186/1472-6947-11-51}
\showDOI{\tempurl}


\bibitem[\protect\citeauthoryear{K{\"o}pcke, Thor, and Rahm}{K{\"o}pcke
  et~al\mbox{.}}{2010}]%
        {kopcke_evaluation_2010}
\bibfield{author}{\bibinfo{person}{Hanna K{\"o}pcke}, \bibinfo{person}{Andreas
  Thor}, {and} \bibinfo{person}{Erhard Rahm}.} \bibinfo{year}{2010}\natexlab{}.
\newblock \showarticletitle{{Evaluation} of {Entity} {Resolution} {Approaches}
  on {Real-world} {Match} {Problems}}.
\newblock \bibinfo{journal}{\emph{PVLDB}} \bibinfo{volume}{3},
  \bibinfo{number}{1} (\bibinfo{year}{2010}), \bibinfo{pages}{484--493}.
\newblock
\showISSN{2150-8097}
\urldef\tempurl%
\url{https://doi.org/10.14778/1920841.1920904}
\showDOI{\tempurl}


\bibitem[\protect\citeauthoryear{Li and Kanoulas}{Li and Kanoulas}{2017}]%
        {li_active_2017}
\bibfield{author}{\bibinfo{person}{Dan Li} {and} \bibinfo{person}{Evangelos
  Kanoulas}.} \bibinfo{year}{2017}\natexlab{}.
\newblock \showarticletitle{Active Sampling for Large-scale Information
  Retrieval Evaluation}. In \bibinfo{booktitle}{\emph{Proceedings of the 2017
  ACM on Conference on Information and Knowledge Management}} (Singapore,
  Singapore) \emph{(\bibinfo{series}{CIKM '17})}. \bibinfo{publisher}{ACM},
  \bibinfo{address}{New York, NY, USA}, \bibinfo{pages}{49--58}.
\newblock
\showISBNx{978-1-4503-4918-5}
\urldef\tempurl%
\url{https://doi.org/10.1145/3132847.3133015}
\showDOI{\tempurl}


\bibitem[\protect\citeauthoryear{Marchant and Rubinstein}{Marchant and
  Rubinstein}{2017}]%
        {marchant_search_2017}
\bibfield{author}{\bibinfo{person}{Neil~G. Marchant} {and}
  \bibinfo{person}{Benjamin I.~P. Rubinstein}.}
  \bibinfo{year}{2017}\natexlab{}.
\newblock \showarticletitle{In {Search} of an {Entity} {Resolution} {OASIS}:
  {Optimal} {Asymptotic} {Sequential} {Importance} {Sampling}}.
\newblock \bibinfo{journal}{\emph{Proc. VLDB Endow.}} \bibinfo{volume}{10},
  \bibinfo{number}{11} (\bibinfo{date}{Aug.} \bibinfo{year}{2017}),
  \bibinfo{pages}{1322--1333}.
\newblock
\showISSN{2150-8097}
\urldef\tempurl%
\url{https://doi.org/10.14778/3137628.3137642}
\showDOI{\tempurl}


\bibitem[\protect\citeauthoryear{Marin, Pudlo, and Sedki}{Marin
  et~al\mbox{.}}{2019}]%
        {marin_consistency_2019}
\bibfield{author}{\bibinfo{person}{Jean-Michel Marin}, \bibinfo{person}{Pierre
  Pudlo}, {and} \bibinfo{person}{Mohammed Sedki}.}
  \bibinfo{year}{2019}\natexlab{}.
\newblock \showarticletitle{Consistency of adaptive importance sampling and
  recycling schemes}.
\newblock \bibinfo{journal}{\emph{Bernoulli}} \bibinfo{volume}{25},
  \bibinfo{number}{3} (\bibinfo{date}{Aug.} \bibinfo{year}{2019}),
  \bibinfo{pages}{1977--1998}.
\newblock
\urldef\tempurl%
\url{https://doi.org/10.3150/18-BEJ1042}
\showDOI{\tempurl}


\bibitem[\protect\citeauthoryear{Minka}{Minka}{1999}]%
        {minka_dirichlet-tree_1999}
\bibfield{author}{\bibinfo{person}{Tom Minka}.}
  \bibinfo{year}{1999}\natexlab{}.
\newblock \bibinfo{booktitle}{\emph{The {Dirichlet-tree} {Distribution}}}.
\newblock \bibinfo{type}{{T}echnical {R}eport}.
  \bibinfo{institution}{Justsystem Pittsburgh Research Center}.
\newblock


\bibitem[\protect\citeauthoryear{Mozafari, Sarkar, Franklin, Jordan, and
  Madden}{Mozafari et~al\mbox{.}}{2014}]%
        {mozafari_scaling_2014}
\bibfield{author}{\bibinfo{person}{Barzan Mozafari}, \bibinfo{person}{Purna
  Sarkar}, \bibinfo{person}{Michael Franklin}, \bibinfo{person}{Michael
  Jordan}, {and} \bibinfo{person}{Samuel Madden}.}
  \bibinfo{year}{2014}\natexlab{}.
\newblock \showarticletitle{Scaling up {Crowd-Sourcing} to {Very} {Large}
  {Datasets}: {A} {Case} for {Active} {Learning}}.
\newblock \bibinfo{journal}{\emph{Proc. VLDB Endow.}} \bibinfo{volume}{8},
  \bibinfo{number}{2} (\bibinfo{date}{Oct.} \bibinfo{year}{2014}),
  \bibinfo{pages}{125--136}.
\newblock
\showISSN{2150-8097}
\urldef\tempurl%
\url{https://doi.org/10.14778/2735471.2735474}
\showDOI{\tempurl}


\bibitem[\protect\citeauthoryear{Oh and Berger}{Oh and Berger}{1992}]%
        {oh_adaptive_1992}
\bibfield{author}{\bibinfo{person}{Man-Suk Oh} {and} \bibinfo{person}{James~O.
  Berger}.} \bibinfo{year}{1992}\natexlab{}.
\newblock \showarticletitle{Adaptive importance sampling in monte carlo
  integration}.
\newblock \bibinfo{journal}{\emph{Journal of Statistical Computation and
  Simulation}} \bibinfo{volume}{41}, \bibinfo{number}{3-4}
  (\bibinfo{year}{1992}), \bibinfo{pages}{143--168}.
\newblock
\urldef\tempurl%
\url{https://doi.org/10.1080/00949659208810398}
\showDOI{\tempurl}


\bibitem[\protect\citeauthoryear{Portier and Delyon}{Portier and
  Delyon}{2018}]%
        {portier_asymptotic_2018}
\bibfield{author}{\bibinfo{person}{François Portier} {and}
  \bibinfo{person}{Bernard Delyon}.} \bibinfo{year}{2018}\natexlab{}.
\newblock \showarticletitle{Asymptotic optimality of adaptive importance
  sampling}.
\newblock In \bibinfo{booktitle}{\emph{Advances in {Neural} {Information}
  {Processing} {Systems} 31}}, \bibfield{editor}{\bibinfo{person}{S.~Bengio},
  \bibinfo{person}{H.~Wallach}, \bibinfo{person}{H.~Larochelle},
  \bibinfo{person}{K.~Grauman}, \bibinfo{person}{N.~Cesa-Bianchi}, {and}
  \bibinfo{person}{R.~Garnett}} (Eds.). \bibinfo{publisher}{Curran Associates,
  Inc.}, \bibinfo{pages}{3138--3148}.
\newblock


\bibitem[\protect\citeauthoryear{{Porto Seguro}}{{Porto Seguro}}{2017}]%
        {porto_kaggle}
\bibfield{author}{\bibinfo{person}{{Porto Seguro}}.}
  \bibinfo{year}{2017}\natexlab{}.
\newblock \bibinfo{title}{Porto {Seguro's} {Safe} {Driver} {Prediction}}.
\newblock
  \bibinfo{howpublished}{\url{https://www.kaggle.com/c/porto-seguro-safe-driver-prediction}}.
\newblock
\newblock
\shownote{Accessed: Dec 2019.}


\bibitem[\protect\citeauthoryear{{Pozzolo}, {Caelen}, {Johnson}, and
  {Bontempi}}{{Pozzolo} et~al\mbox{.}}{2015}]%
        {pozzolo_calibrating_2015}
\bibfield{author}{\bibinfo{person}{A.~D. {Pozzolo}}, \bibinfo{person}{O.
  {Caelen}}, \bibinfo{person}{R.~A. {Johnson}}, {and} \bibinfo{person}{G.
  {Bontempi}}.} \bibinfo{year}{2015}\natexlab{}.
\newblock \showarticletitle{Calibrating {Probability} with {Undersampling} for
  {Unbalanced} {Classification}}. In \bibinfo{booktitle}{\emph{2015 IEEE
  Symposium Series on Computational Intelligence}}. \bibinfo{pages}{159--166}.
\newblock
\urldef\tempurl%
\url{https://doi.org/10.1109/SSCI.2015.33}
\showDOI{\tempurl}


\bibitem[\protect\citeauthoryear{Reddy and Vinzamuri}{Reddy and
  Vinzamuri}{2014}]%
        {reddy_survey_2010}
\bibfield{author}{\bibinfo{person}{Chandan~K. Reddy} {and}
  \bibinfo{person}{Bhanukiran Vinzamuri}.} \bibinfo{year}{2014}\natexlab{}.
\newblock \bibinfo{booktitle}{\emph{A {Survey} of {Partitional} and
  {Hierarchical} {Clustering} {Algorithms}} (\bibinfo{edition}{1st} ed.)}.
\newblock \bibinfo{publisher}{Chapman \& Hall\slash CRC},
  \bibinfo{pages}{87--110}.
\newblock
\showISBNx{1466558210}
\urldef\tempurl%
\url{https://doi.org/10.1201/9781315373515-4}
\showDOI{\tempurl}


\bibitem[\protect\citeauthoryear{RIDDLE}{RIDDLE}{2003}]%
        {riddle_datasets}
RIDDLE \bibinfo{year}{2003}\natexlab{}.
\newblock \bibinfo{title}{{Duplicate Detection}, {Record Linkage}, and
  {Identity Uncertainty}: {Datasets}}.
\newblock
  \bibinfo{howpublished}{\url{http://www.cs.utexas.edu/users/ml/riddle/data.html}}.
\newblock
\newblock
\shownote{Accessed: Dec 2016.}


\bibitem[\protect\citeauthoryear{Rubinstein and Kroese}{Rubinstein and
  Kroese}{2016}]%
        {rubinstein_simulation_2016}
\bibfield{author}{\bibinfo{person}{Reuven~Y. Rubinstein} {and}
  \bibinfo{person}{Dirk~P. Kroese}.} \bibinfo{year}{2016}\natexlab{}.
\newblock \bibinfo{booktitle}{\emph{Simulation and the {Monte} {Carlo}
  {Method}}}.
\newblock \bibinfo{publisher}{John Wiley \& Sons, Ltd}.
\newblock
\showISBNx{9781118631980}
\urldef\tempurl%
\url{https://doi.org/10.1002/9781118631980}
\showDOI{\tempurl}


\bibitem[\protect\citeauthoryear{Sawade, Landwehr, Bickel, and Scheffer}{Sawade
  et~al\mbox{.}}{2010b}]%
        {sawade_active_2010a}
\bibfield{author}{\bibinfo{person}{Christoph Sawade}, \bibinfo{person}{Niels
  Landwehr}, \bibinfo{person}{Steffen Bickel}, {and} \bibinfo{person}{Tobias
  Scheffer}.} \bibinfo{year}{2010}\natexlab{b}.
\newblock \showarticletitle{Active {Risk} {Estimation}}. In
  \bibinfo{booktitle}{\emph{Proceedings of the 27th International Conference on
  International Conference on Machine Learning}} (Haifa, Israel)
  \emph{(\bibinfo{series}{ICML’10})}. \bibinfo{publisher}{Omnipress},
  \bibinfo{address}{Madison, WI, USA}, \bibinfo{pages}{951--958}.
\newblock
\showISBNx{9781605589077}


\bibitem[\protect\citeauthoryear{Sawade, Landwehr, and Scheffer}{Sawade
  et~al\mbox{.}}{2010a}]%
        {sawade_active_2010}
\bibfield{author}{\bibinfo{person}{Christoph Sawade}, \bibinfo{person}{Niels
  Landwehr}, {and} \bibinfo{person}{Tobias Scheffer}.}
  \bibinfo{year}{2010}\natexlab{a}.
\newblock \showarticletitle{Active {Estimation} of {F}-{Measures}}.
\newblock In \bibinfo{booktitle}{\emph{Advances in {Neural} {Information}
  {Processing} {Systems} 23}}, \bibfield{editor}{\bibinfo{person}{J.~D.
  Lafferty}, \bibinfo{person}{C.~K.~I. Williams},
  \bibinfo{person}{J.~Shawe-Taylor}, \bibinfo{person}{R.~S. Zemel}, {and}
  \bibinfo{person}{A.~Culotta}} (Eds.). \bibinfo{publisher}{Curran Associates,
  Inc.}, \bibinfo{pages}{2083--2091}.
\newblock


\bibitem[\protect\citeauthoryear{Schnabel, Swaminathan, Frazier, and
  Joachims}{Schnabel et~al\mbox{.}}{2016}]%
        {schnabel_unbiased_2016}
\bibfield{author}{\bibinfo{person}{Tobias Schnabel}, \bibinfo{person}{Adith
  Swaminathan}, \bibinfo{person}{Peter~I. Frazier}, {and}
  \bibinfo{person}{Thorsten Joachims}.} \bibinfo{year}{2016}\natexlab{}.
\newblock \showarticletitle{Unbiased {Comparative} {Evaluation} of {Ranking}
  {Functions}}. In \bibinfo{booktitle}{\emph{Proceedings of the 2016 ACM
  International Conference on the Theory of Information Retrieval}} (Newark,
  Delaware, USA) \emph{(\bibinfo{series}{ICTIR '16})}.
  \bibinfo{publisher}{ACM}, \bibinfo{address}{New York, NY, USA},
  \bibinfo{pages}{109--118}.
\newblock
\showISBNx{978-1-4503-4497-5}
\urldef\tempurl%
\url{https://doi.org/10.1145/2970398.2970410}
\showDOI{\tempurl}


\bibitem[\protect\citeauthoryear{Schultheis, Qaraei, Gupta, and
  Babbar}{Schultheis et~al\mbox{.}}{2020}]%
        {schultheis2020unbiased}
\bibfield{author}{\bibinfo{person}{Erik Schultheis},
  \bibinfo{person}{Mohammadreza Qaraei}, \bibinfo{person}{Priyanshu Gupta},
  {and} \bibinfo{person}{Rohit Babbar}.} \bibinfo{year}{2020}\natexlab{}.
\newblock \bibinfo{title}{Unbiased Loss Functions for Extreme Classification
  With Missing Labels}.
\newblock
\newblock
\showeprint[arxiv]{2007.00237}~[stat.ML]


\bibitem[\protect\citeauthoryear{Settles}{Settles}{2009}]%
        {settles_active_2009}
\bibfield{author}{\bibinfo{person}{Burr Settles}.}
  \bibinfo{year}{2009}\natexlab{}.
\newblock \bibinfo{booktitle}{\emph{{Active Learning Literature Survey}}}.
\newblock \bibinfo{type}{{T}echnical {R}eport}.
  \bibinfo{institution}{University of Wisconsin-Madison Department of Computer
  Sciences}.
\newblock
\urldef\tempurl%
\url{http://digital.library.wisc.edu/1793/60660}
\showURL{%
\tempurl}


\bibitem[\protect\citeauthoryear{Vaart}{Vaart}{1998}]%
        {vaart_delta_1998}
\bibfield{author}{\bibinfo{person}{A.~W. van~der Vaart}.}
  \bibinfo{year}{1998}\natexlab{}.
\newblock \bibinfo{booktitle}{\emph{Delta {Method}}}.
\newblock \bibinfo{publisher}{Cambridge University Press},
  \bibinfo{pages}{25–--34}.
\newblock
\urldef\tempurl%
\url{https://doi.org/10.1017/CBO9780511802256.004}
\showDOI{\tempurl}


\bibitem[\protect\citeauthoryear{Wei, Li, Cao, Ou, and Chen}{Wei
  et~al\mbox{.}}{2013}]%
        {wei_effective_2013}
\bibfield{author}{\bibinfo{person}{Wei Wei}, \bibinfo{person}{Jinjiu Li},
  \bibinfo{person}{Longbing Cao}, \bibinfo{person}{Yuming Ou}, {and}
  \bibinfo{person}{Jiahang Chen}.} \bibinfo{year}{2013}\natexlab{}.
\newblock \showarticletitle{Effective detection of sophisticated online banking
  fraud on extremely imbalanced data}.
\newblock \bibinfo{journal}{\emph{World Wide Web}} \bibinfo{volume}{16},
  \bibinfo{number}{4} (\bibinfo{year}{2013}), \bibinfo{pages}{449--475}.
\newblock
\showISSN{1573-1413}
\urldef\tempurl%
\url{https://doi.org/10.1007/s11280-012-0178-0}
\showDOI{\tempurl}


\bibitem[\protect\citeauthoryear{Welinder, Welling, and Perona}{Welinder
  et~al\mbox{.}}{2013}]%
        {welinder_lazy_2013}
\bibfield{author}{\bibinfo{person}{P. Welinder}, \bibinfo{person}{M. Welling},
  {and} \bibinfo{person}{P. Perona}.} \bibinfo{year}{2013}\natexlab{}.
\newblock \showarticletitle{A {Lazy} {Man's} {Approach} to {Benchmarking}:
  {Semisupervised} {Classifier} {Evaluation} and {Recalibration}}. In
  \bibinfo{booktitle}{\emph{CVPR}}. \bibinfo{pages}{3262--3269}.
\newblock
\urldef\tempurl%
\url{https://doi.org/10.1109/CVPR.2013.419}
\showDOI{\tempurl}


\end{thebibliography}
